\documentclass[twoside,11pt]{article}

\usepackage{subcaption}
\usepackage{booktabs}

\usepackage{jmlr2e}

\usepackage{amsmath}
\usepackage{hyperref}
\usepackage{units}

\usepackage[latin1]{inputenc}
\usepackage[T1]{fontenc}

\usepackage[algo2e,linesnumbered]{algorithm2e}
\usepackage{float}
\usepackage{placeins}

\usepackage{bm}

\usepackage{lastpage}

\newtheorem{problem}[theorem]{Problem}
\DeclareMathOperator*{\argmax}{arg\,max}

\ShortHeadings{Guided Visual Exploration of Relations in Data Sets}{Puolam\"aki, Oikarinen and Henelius}
\firstpageno{1}

\begin{document}
\jmlrheading{22}{2021}{1-\pageref{LastPage}}{5/19; Revised
4/21}{4/21}{19-364}{Kai Puolam\"aki, Emilia Oikarinen and Andreas Henelius}

\title{Guided Visual Exploration of Relations in Data Sets}

\author{\name Kai Puolam\"aki \email kai.puolamaki@helsinki.fi \\
\addr  Institute for Atmospheric and Earth System Research  \\ Department of Computer Science \\
P.O. Box 68 \\
FI-00014 University of Helsinki, Helsinki, Finland 
\AND
\name Emilia Oikarinen \email emilia.oikarinen@helsinki.fi \\
\addr Department of Computer Science\\
P.O. Box 68 \\
FI-00014 University of Helsinki, Helsinki, Finland 
\AND
\name Andreas Henelius \email andreas.henelius@op.fi \\
\addr 
OP Financial Group \\
Gebhardinaukio 1, FI-00510 Helsinki, Finland\\ \\
Department of Computer Science\\
P.O. Box 68 \\
FI-00014 University of Helsinki, Helsinki, Finland \\
}

\editor{David Blei}
\maketitle

\begin{abstract}%
Efficient explorative data analysis systems must take into account
both what a user knows and wants to know. This paper proposes a
principled framework for interactive visual exploration of relations
in data, through views most informative given the user's current
knowledge and objectives. The user can input pre-existing knowledge of
relations in the data and also formulate specific exploration
interests, which are then taken into account in the exploration. The idea is to
steer the exploration process towards the interests of the user,
instead of showing uninteresting or already known relations. The
user's knowledge is modelled by a distribution over data sets
parametrised by subsets of rows and columns of data, called tile
constraints. We provide a computationally efficient implementation of
this concept based on constrained randomisation. Furthermore, we
describe a novel dimensionality reduction method for finding the views
most informative to the user, which at the limit of no background
knowledge and with generic objectives reduces to PCA. We show that the
method is suitable for interactive use and is robust to noise,
outperforms standard projection pursuit visualisation methods, and
gives understandable and useful results in analysis of real-world
data. We provide an open-source implementation of the framework.
\end{abstract}


% ----------------------------------------------------------------------
\section{Introduction}\label{sec:intro}
% ------------------------------------------------------------
Exploratory data analysis \citep{tukey:1977}, often performed
interactively, is an established approach for learning about patterns
in a data set prior to more formal analyses. Humans are able to easily
identify patterns in the data visually, even when the patterns are
complex and difficult to model algorithmically. Visual data
exploration is hence a powerful tool for exploring patterns in the
data and a multitude of visual exploration systems have been designed
for this purpose over the years. Let us now consider some general
requirements for an \emph{efficient} visual exploration system.

\begin{enumerate}
\item[(i)] The system must \emph{take into account the user's
  knowledge} of the data, which iteratively accumulates during exploration.
\item[(ii)] The user must be shown \emph{informative views} of the
  data given the user's current knowledge.
\item[(iii)] The user must be able to \emph{steer the exploration} in
  order to answer specific questions.
\end{enumerate}

 Despite the long history of visual exploration systems, they still
 lack a principled approach with respect to these general
 requirements. In this paper we address several shortcomings related
 to these requirements. Specifically, our goal and main contribution
 is to devise a framework for human-guided data exploration by
 modelling the user's background knowledge and objectives, and using
 these to provide the user with the most informative views of the
 data.

Our contribution consists of three main parts: (i) a framework for
modelling and incorporating the user's background knowledge of the
data that can be iteratively updated, (ii) finding the most
informative views of the data, and (iii) a solution allowing the user
to steer the visual data exploration process so that specific
hypotheses formulated by the user can be answered. The first and third
contribution are general, while the second one, that is, finding the
most informative views of the data, is specific to a particular data
type. In this paper we focus on data items that can be represented as
real-valued vectors of attribute values. This paper extends our
earlier works: preprint \citep*{puolamaki2018human} and 
\citep{ecml-pkdd:tiler}, the latter of which only considers axis-aligned projections
of the data and does not take advantage of the dimensionality reduction
method presented in this work.

We next discuss the relation of our present work to existing
literature on exploratory data analysis. Our first contribution is
related to \emph{iterative data mining} \citep{hanhijarvi:2009} which
is a paradigm where patterns already discovered by the user are taken
into account as constraints during subsequent exploration. In brief,
this works as follows. The user explores the data and observes a
pattern in a view. The user marks the observed pattern as known in the
exploration system. The system then takes this newly added pattern, as
well as all other previously added patterns, into account when
constructing the next view shown to the user. The goal is to prevent
the system from showing already known information to the user
again. This concept of iterative pattern discovery is also central to
the data mining framework presented by \citeauthor{debie:2011a}
\citeyearpar{debie:2011a,debie:2011b,debie2013}, where the user's
current knowledge (or beliefs) of the data is modelled as a
probability distribution over data sets. This distribution is
updated iteratively during the exploration phase as the user discovers
new patterns. Our work has been motivated by
\citet{puolamaki2010,puolamaki:2016, kang:2016b} and
\citet*{puolamaki2017}, where these concepts have been successfully
applied in visual exploratory data analysis such that the user is
visually shown a view of the data which is maximally informative given
the user's current knowledge. Visual interactive exploration has also
been applied in different contexts, for example, in item-set mining
and subgroup discovery \citep{boley2013, dzyuba2013, vanleeuwen2015,
  paurat2014}, information retrieval \citep{ruotsalo2015}, and network
analysis~\citep{chau2011}.

Concerning our second contribution, solving the problem of determining
which views of the data are \emph{maximally informative} to the user
(and hence interesting) has been approached in terms of, for example,
different projections and measures of interestingness
\citep{debie:2016:a, kang2016, vartak:2015:a, Kang2020}. 
Constraints have also
been used to assess the significance of data mining results, for
example, in pattern mining \citep{lijffijt2014} or in investigating
spatio-temporal relations \citep{chirigati:2016:a}. We observe,
however, that a view maximally informative to the user, is a view that
contrasts the most with the user's current knowledge. Hence, this kind of a
view is maximally ``surprising'' to the user with respect to his or
her current knowledge.

However, always showing maximally informative views to the user leads
to a problem, which can be seen as one of the major shortcomings of
previous work on iterative data mining and applications to visual
exploratory data analysis. By definition, maximally informative views
given the user's existing knowledge will be surprising. Because the
user is not able to control the path that the exploration takes, it is
difficult to investigate specific hypotheses concerning the data or to
steer the exploration process. Traditional iterative data mining
hence suffers from a \emph{navigational problem}
\citep{puolamaki2010}. Our third contribution is to solve this
navigational problem by incorporating both the user's knowledge of the
data, and different hypotheses concerning the data into the background
distribution. It often is the case that the user has some
pre-existing exploration objectives before starting the analysis, or
the user develops specific hypotheses during the exploration
phase. This navigational aspect in the exploration process has, as far
as we are aware of, not been addressed previously, and we believe that
the contribution we make in this area is highly important for any real
interactive iterative data analysis framework.

\begin{figure}[t]
  \centering
  \includegraphics[width = 0.75\textwidth]{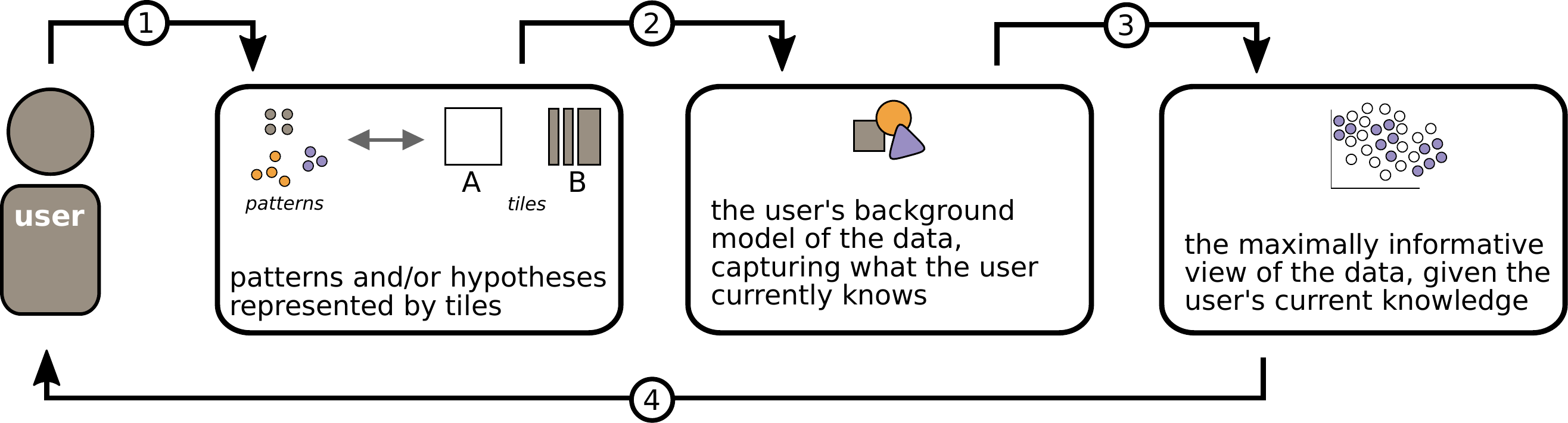}
  \caption{Overview of the exploration framework. (1) The user inputs
    his or her current knowledge and exploration objectives.  (2) A
    background distribution, which captures what the user currently
    knows is constructed.  (3) The most informative view of the data
    with respect to the background distribution and the user's
    objectives is computed. (4) The user observes the data in the
    view, recognises relations and inputs these into the background
    distribution.  The iterative data analysis process continues from
    (1).  }
  \label{fig:framework}
\end{figure}

\begin{figure}[t]
  \centering
  \begin{tabular}{ccc}
    \includegraphics[width = 0.2\textwidth]{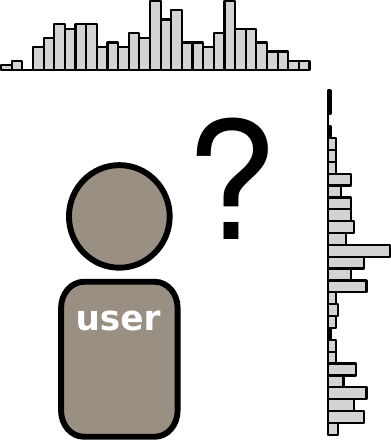} \hspace*{2em} &
    \includegraphics[width = 0.2\textwidth]{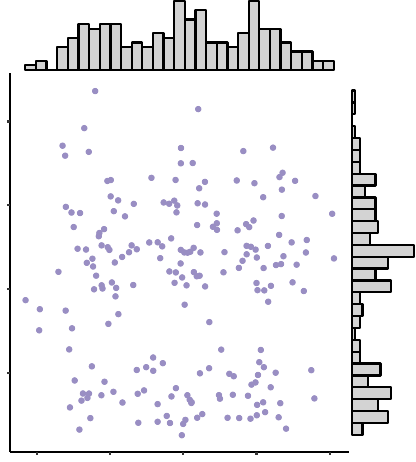} \hspace*{2em} &
    \includegraphics[width = 0.2\textwidth]{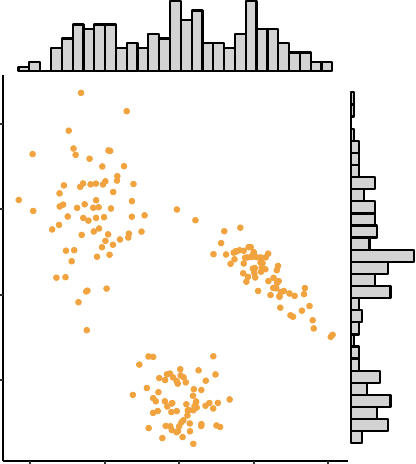} \\
    \multicolumn{3}{c}{\vspace*{0.1em}} \\
    (a) \hspace*{5em} & (b) \hspace*{5em} & (c)
  \end{tabular}
  \caption{Modelling the user's background knowledge. (a)
    Initially the user knows nothing about the data except the
    marginals. (b) A possible data sample which adheres to the user's
    current knowledge of the data. In this case, the data sample
    corresponds to a situation where all attributes have been permuted
    independently. (c) A sample corresponding to what the user could
    potentially learn from the data (in this case, this is the real
    data with all relations intact).}
  \label{fig:background}
\end{figure}

Our framework is sketched in Figure \ref{fig:framework}.  More
formally, as in \citet{lijffijt2014}, we denote the original data set
by $X$ and the set of all possible data sets by $\Omega$. We further
define a set of {\em constraints} ${\mathcal{C}}$. A constraint
$t\in{\mathcal{C}}$ is simply a subset of all possible data sets which
always also includes the original data set, that is, $X\in t\subseteq
\Omega$ is satisfied for all $t\in{\mathcal{C}}$.  Any set of
constraints ${\cal{T}}\subseteq{\mathcal{C}}$ can be used to define a
subset of data sets $\Omega_{\cal{T}}\subseteq\Omega$ that satisfy all
of the constraints in ${\cal{T}}$ by
$\Omega_{\cal{T}}=\cap_{t\in{\cal{T}}}{t}$, with $\Omega_\emptyset$
defined as $\Omega_\emptyset=\Omega$.

We assume that the user observes a set of {\em relations} such as
correlations, cluster structures etc.  from the data, as later defined
by Definition~\ref{def:relation}.  A constraint---or a set of
constraints---can either preserve or break a relation.  If the user
observes that some relations are preserved in the data the user can
infer that the data obeys the constraints that preserve these
relations.

In this paper we assume that the user's knowledge can be parametrised
by a set of constraints ${\cal T}_u$ and by a uniform distribution
over data sets in $\Omega_{{\cal{T}}_u}$, with the probability of the
data sets in the complement $\Omega\setminus\Omega_{{\cal{T}}_u}$
being zero.  We call this uniform distribution a {\em background
  distribution}, which describes the probabilities the user gives for
different possible data sets.  Intuitively, the constraints denote the
relations (or patterns) in the data that the user already is aware
of. In the user's mind, any data set that is not contradictory with
any of these constraints is equiprobable, while data sets
contradicting with any of the constraints have zero probability. In the
absence of constraints, that is, when the user knows nothing, the
user's knowledge is described by the background distribution
corresponding to the situation that the user considers all possible
data sets equally likely, as shown in Figure~\ref{fig:background}. Our objective is that the system would not show the user relations or patterns that the user is already aware of, as parametrised by constraints in ${\cal T}_u$.

Now, out of all possible {\em views} of the data (such as scatter
plots over different coordinate axes) the most informative view should be
the one that---according to some measure---shows the maximal
difference between the data set $X$ and a data set sampled from the
background distribution as in Figure~\ref{fig:informativeview}. When
looking at this view, the user may learn more about the relations in
the data and can add the gained knowledge as new constraints in
${\cal{T}}_u$, after which a new maximally informative view can be
produced. This iterative process continues until the user has learned
all he or she wants to know about the data; this is the approach
taken, for example, in \citet{puolamaki:2016} and
\citet*{puolamaki2017}. As already mentioned, the drawback of the
approach is that each new view is by definition also maximally
surprising to the user and there is no way to guide the
exploration towards the user's interests.

In this paper we complement this framework by using the constraints to
parametrise, in addition to the user's knowledge, what the user wants
to know about the data. We do this by defining a new set of
constraints, denoted by ${\cal{T}}_H\subseteq{\cal{C}}$ which defines
the relations that are of the interest to the user.  We further define
a set of constraints ${\cal{T}}_{H'}\subseteq{\cal{C}}$ which defines
the relations that are of no interest to the user.  Instead of
comparing the data and the background distribution, as earlier, we
find a view that shows the maximal difference between samples from the
uniform distributions from $\Omega_{{\cal{T}}_u\cup{\cal{T}}_{H'}}$
and $\Omega_{{\cal{T}}_u\cup{\cal{T}}_{H'}\cup {\cal{T}}_H}$,
respectively.  Notice that if we are interested in all constraints,
i.e., ${\cal{T}}_H={\cal{C}}$ and ${\cal{T}}_{H'}=\emptyset$, this new
formulation reduces to the earlier approach of \citet{puolamaki:2016}
and \citet*{puolamaki2017}, at least if all of the constraints
together allow only the original data set, or
$\Omega_{{\cal{C}}}=\{X\}$.

\begin{figure}[t]
  \centering
  \begin{tabular}{ccc}
    \textbf{Informative view} & & \textbf{Uninformative view} \\
    & & \\
    \includegraphics[width = 0.4\textwidth]{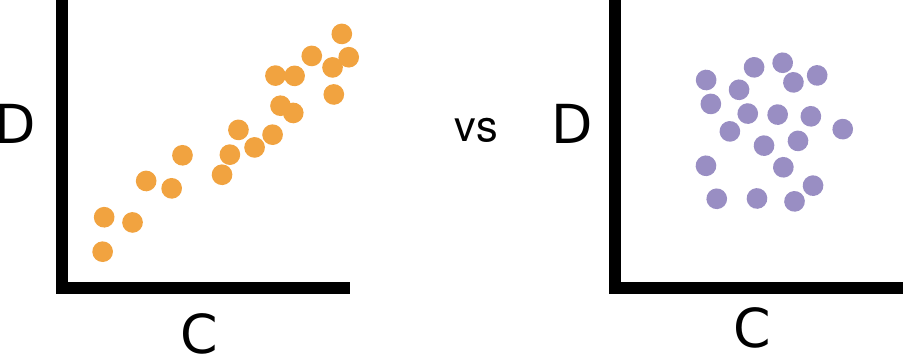} &
    \hspace*{3em} &
    \includegraphics[width = 0.4\textwidth]{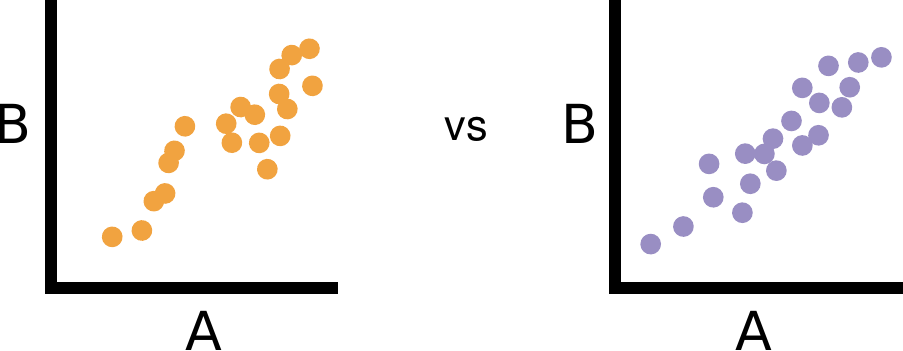}
  \end{tabular}
  \caption{Two data samples (shown in orange and purple) in two
    projections.  A view is informative if the data samples differ,
    which means that the user can learn something new about the
    data. In the left view the data samples differ with respect to the
    attributes $C$ and $D$ and the view is informative. In contrast,
    in the right view the data samples are essentially the same and
    the user learns nothing new about the relation between attributes
    $A$ and $B$, and consequently the view is uninformative.}
  \label{fig:informativeview}
\end{figure}

The advantage of the new formulation is that by expressing the user's
knowledge and the user's objectives using the same parametrisation we
can in an elegant way formalise the data exploration process as
finding views that show differences between two distributions, or
between samples from two distributions.

The above discussion is generic for all classes of data sets and
constraints. However, in the remainder of the paper we assume that the
data set $X$ is a data table with rows corresponding to data items and
columns to attributes, and the set $\Omega$ consists of all data sets
that can be obtained by randomly permuting the columns of $X$. The
constraints in ${\cal{C}}$ are parametrised by {\em tiles} containing
subsets of rows and columns, respectively.  Samples from the
constrained distribution can then be obtained efficiently by
permutations, as described later in Section \ref{sec:methods}, which
also gives the exact definitions for the concepts mentioned above. The
views considered in this paper are scatterplots of the data obtained
by linear projections that show the maximal differences of the
distributions described above.

\paragraph{Contributions}
In summary, our contributions are:
\begin{enumerate}
\item[(1)] a computationally efficient formulation and implementation
  of the user's background knowledge of the data and objectives (which
  we here call hypotheses) using constrained randomisation,
\item[(2)] a dimensionality reduction method for finding the view most
  informative to the user, and
\item[(3)] an experimental evaluation that supports that our approach
  is fast, robust, and produces easily understandable results.
\end{enumerate}

\paragraph{Outline}
The rest of this paper is organised as follows.  Our framework is
formalised in Section~\ref{sec:methods}, where we describe how to
model the user's background knowledge of the data, how data
exploration objectives are formulated and updated, and how maximally
informative views are determined. In Section~\ref{sec:experiments} we
empirically evaluate our framework, by considering both computational
efficiency and robustness against noise, as well as provide use cases
of user-guided exploration. We conclude the paper with a discussion in
Section~\ref{sec:conclusions}.

% ------------------------------------------------------------
\section{Methods}\label{sec:methods}
% ------------------------------------------------------------

Let $X$ be an $n \times m$ data matrix (data set). Here $X(i,j)$
denotes the $i$th element in column $j$. Each column $X(\cdot, j),\ j
\in [m]$, is an \emph{attribute} in the data set, where we use the shorthand
$[m] = \{1, \ldots, m\}$. Let $D$ be a finite set of domains (for
example, continuous or categorical) and let $D(j)$ denote the domain
of $X(\cdot, j)$. Also let $X(i, j) \in D(j)$ for all $i \in [n]$ and
$j \in [m]$, that is, all elements in a column belong to the same
domain but different columns can have different domains. The
derivations in Sections~\ref{sec:bg}
and~\ref{sec:formulatinghypotheses} are generic with respect to
domains, but in Section~\ref{sec:view} we consider only real numbers,
that is, $D(j) \subseteq \mathbb{R}$ for all $j\in[m]$.

\subsection{Permutations and Tile Constraints}
\label{sec:bg}
We proceed to introduce the permutation-based sampling method and tile
constraints which are used to constrain the sampled distributions as
well as to express the user's background knowledge and objectives
(hypotheses). The distributions are constructed so that in the absence
of constraints (tiles) the marginal distributions of the attributes
are preserved.

We define a \emph{permutation} of the data matrix $X$ as follows.
\begin{definition}[Permutation]\label{def:permutation}{
  Let ${\mathcal P}_n$ denote the set of permutation functions of length
  $n$ such that
  $\pi:[n]\to[n]$ is a bijection for all
  $\pi\in{\mathcal P}_n$, and denote by
  $\Pi=\left(\pi_1,\ldots,\pi_m\right)\in{\mathcal P}^m_n$ the vector of
  column-specific permutations. A permutation $\Pi(X)=\widehat X$ of a $n \times m$ data
  matrix $X$ is then given as $\widehat X(i,j)=X(\pi_j(i),j)$.}
\end{definition}
When permutation functions are sampled uniformly at random, we obtain
a uniform sample from the distribution of data sets where each of the
attributes has the same marginal distribution as the original
data. Hence, given a data set $X$, the set of possible data sets is
$\Omega=\{ \Pi(X) \mid \Pi \in{\mathcal P}^m_n\}$.

\begin{figure}[t]
  \centering
  \begin{subfigure}{\textwidth}
    \begin{minipage}[t]{.45\linewidth}
      \vspace{0pt} \textbf{Unconstrained permutation}\\ All attributes
      $A, \ldots, D$ are permuted independently. This breaks all
      relations between the attributes, but the marginals are preserved.
    \end{minipage}
    \hfill
    \begin{minipage}[t]{.45\linewidth}
      \vspace{0pt}
      \includegraphics[width = \textwidth]{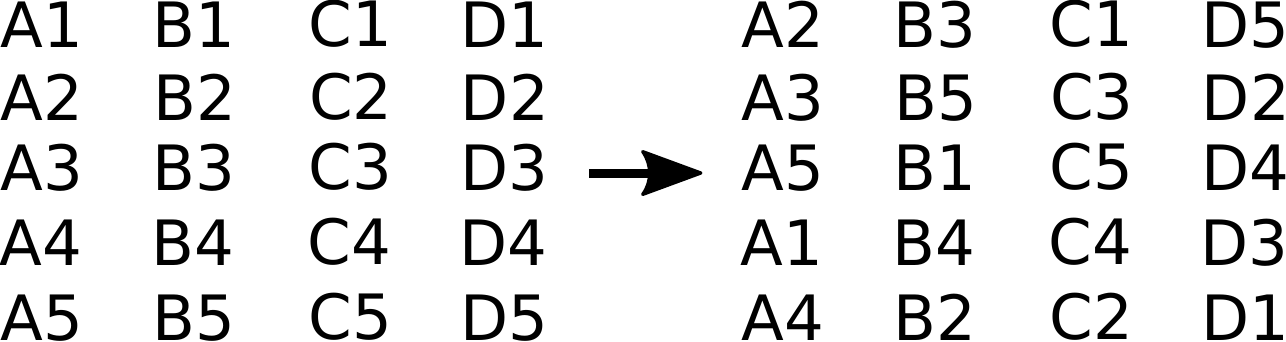}
    \end{minipage}
  \end{subfigure}

  \vspace*{2em}

  \begin{subfigure}{\textwidth}
    \begin{minipage}[t]{.45\linewidth}
      \vspace{0pt} \textbf{Constrained permutation}\\ The permutation
      is constrained with a tile, that is,  attributes in the
       tile are permuted together. The marginals are also preserved.
    \end{minipage}
    \hfill
    \begin{minipage}[t]{.45\linewidth}
      \vspace{0pt}
      \includegraphics[width = \textwidth]{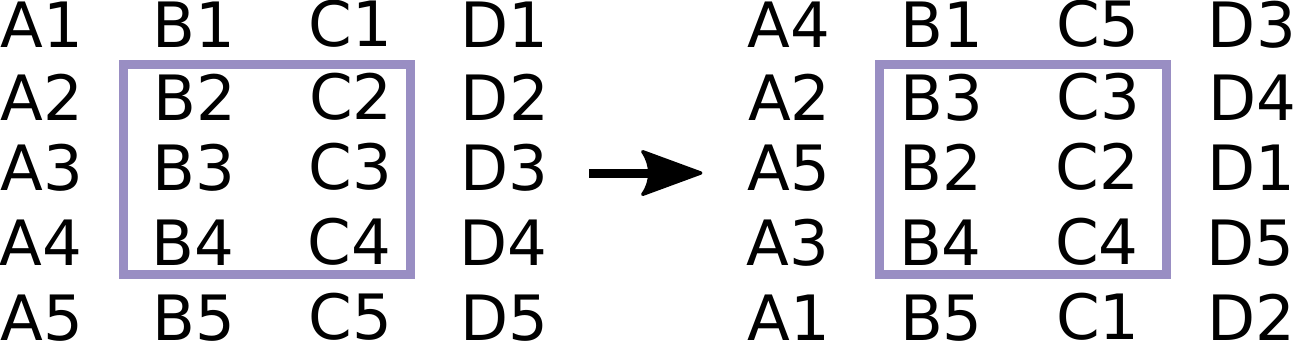}
    \end{minipage}
  \end{subfigure}

  \caption{An unconstrained permutation (top) and a permutation
    constrained with a tile (bottom, tile shown with a solid
    border). The data has four attributes ($A,B,C$, and $D$) and there
    are five data items (rows) in the data set.}
  \label{fig:permutations}
\end{figure}

A \emph{tile} is a tuple $t = (R, C)$, where $R \subseteq [n]$ and $C
\subseteq [m]$.  The tiles considered here are combinatorial (in
contrast to geometric), meaning that rows and columns in the tile do
not need to be consecutive.  In the unconstrained case, there are
$(n!)^m$ allowed vectors of permutations.  We parametrise
distributions using \emph{tile constraints} preserving the relations
in a data matrix $X$ for subsets of rows and columns. The tiles
constrain the set of allowed permutations as follows.
\begin{definition}[Tile constraint]
  \label{def:tileconstraint}
  Given a tile
  $t = (R,C)$ where  $R
\subseteq [n]$ and $C \subseteq [m]$,  a vector of permutations
  $\Pi=\left(\pi_1,\ldots,\pi_m\right)\in{\mathcal P}^m_n$ is allowed by $t$
  iff the following condition is true for all $i\in[n]$, $j\in[m]$,
  and $j'\in[m]$:
$$
i\in R \textrm{ and } j,j'\in C  \\
\implies \pi_j(i)\in R \textrm{ and } \pi_j(i)=\pi_{j'}(i)
.$$
Given a set of tiles $T$, a vector of
permutations $\Pi$ is allowed iff $\Pi$ is allowed by all $t\in T$.
For an empty set of tiles $T=\emptyset$, all permutations in ${\mathcal P}^m_n$
are allowed.
\end{definition}
A tile defines a subset of rows and columns, and the rows in this
subset are permuted by the same permutation function in each column in
the tile. In other words, the relations between the columns inside the
tile are preserved.  Thus, given a data set $X$ and a set of tiles
$T$, the subset of data sets in $\Omega$ that satisfy all of the tile
constraints in $T$ is given by
$$\Omega_T=\{ \Pi(X) \mid \Pi \in{\mathcal P}^m_n\mbox{ and }\Pi\mbox{ is allowed by }T\}.$$
  Notice that the identity permutation
is always an allowed permutation. Figure~\ref{fig:permutations} shows
an example of both unconstrained permutation and permutation
constrained with a tile.

We proceed to define formally what we mean by relations in this paper.
\begin{definition}[Relation]
\label{def:relation}
A relation is a real-valued function $f$ over $n \times m$ data
matrices~$X$.  Given a set of tiles $T$, we say that that $T$
preserves the relation $f$, if $f(X)=f(\Pi(X))$ is satisfied for all
permutations $\Pi$ allowed by $T$.  Otherwise, we say that $T$ breaks
the relation~$f$.
\end{definition}
Thus, we use the term relation to denote any structure in the data
which can be controlled (that is, essentially broken if need be) in
the permutation scheme parametrised by the tile constraints. In
practise, some tolerance could be included in the above definition
for the condition $f(X)=f(\Pi(X))$ instead of exact equivalence. Examples
of relations conforming to the above definition include correlations
between attributes, and cluster structures. For example, for a real
valued data matrix a relation could be defined as a covariance between
columns $a$ and $b$, i.e.,
$f(X)=\sum\nolimits_{i=1}^n{X(i,a)X(i,b)}/n$. A set of tiles
containing a tile $t=([n],\{a,b\})$ preserves this relation. On the
other hand, a set of tiles allowing some of the rows in
columns $a$ and $b$ to be permuted independently breaks the
relation $f$. Another example of a possible relation are the
scagnostics for a scatterplot visualisation \citep{wilkinson:2005:a}.

We make an implicit assumption that if the user observes in the data
that certain (visual) relations are preserved, then the user can
conclude that the data also obeys constraints that preserve those same
relations. The user can then add these constraints to the background
distribution. Note that the relations $f$ correspond to visual
patterns (correlations, cluster structures, outliers etc.) which the
user possibly observes. The relations $f$ are not evaluated
by the computer, but they are part of the user's cognitive
processing of the visualisations. Therefore, in practical
applications, there is usually no need---nor would it be possible---to
define all of the relations $f$ explicitly. For our purposes it suffices
that the user can to a reasonable accuracy match the observed
visual relations in the data to the corresponding constraints.

We use tile constraints to describe the user's knowledge concerning
relations in the data.  As the user views the data he or she can
observe relations and represent these as tile constraints. For
example, the user can mark an observed cluster structure with a tile
involving the data points in the cluster and the relevant
attributes. We denote the set of user-defined tiles by ${\mathcal
  T}_u$. Then, a uniform distribution from $\Omega_{{\cal T}_u}=\{
\Pi(X) \mid \Pi \in{\mathcal P}^m_n\mbox{ and }\Pi\mbox{ is allowed by
}{\cal T}_u\}$ is the {\em background distribution}, which describes
the probabilities the user gives for different possible data sets.

We can now formulate our sampling problem as follows.
\begin{problem}[Sampling problem]
\label{def:samplingproblem}
Given a set of tiles $T$,  draw samples
uniformly at random from vectors of permutations  $\Pi\in{\mathcal P}^m_n$ allowed by $T$.
\end{problem}
The sampling problem is trivial when the tiles are non-overlapping,
since permutations can be done independently within each
non-overlapping tile. However, in the case of overlapping tiles,
multiple constraints can affect the permutation of the same subset of
rows and columns and this issue must be resolved. To this end, we need
to define the equivalence of two sets of tiles, which means that the
same constraints are enforced on the permutations.
\begin{definition}[Equivalence of sets of tiles]
\label{def:equivalence}
Let $T$ and $T'$ be two sets of tiles. $T$ is \emph{equivalent} to
$T'$, if for all vectors of permutations $\Pi\subseteq
\mathcal{P}^m_n$ it holds:
\begin{equation*}
\Pi\mbox{ is allowed by }T\mbox{ iff }\Pi\mbox{ is allowed by }T'.
\end{equation*}
\end{definition}

\begin{figure}[t]
  \centering \includegraphics[width = 0.45\textwidth]{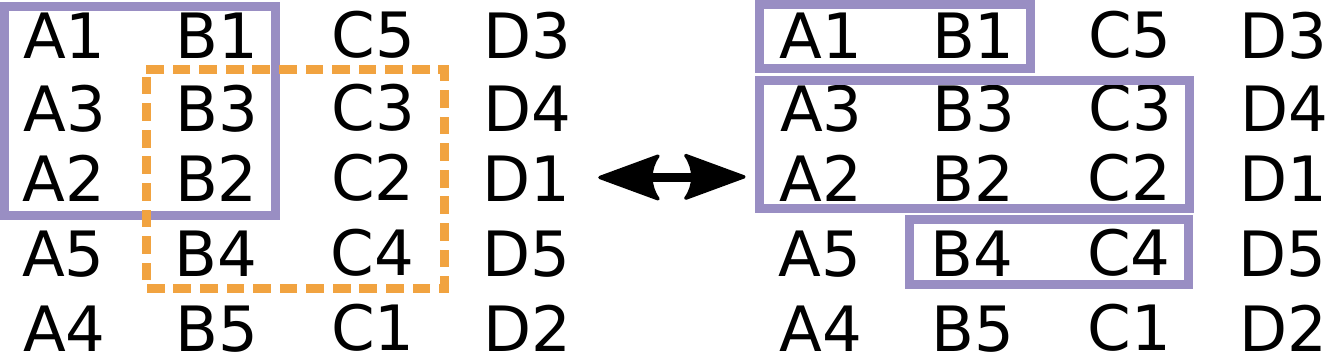}
  \caption{Merging the two overlapping purple (solid border) and
    orange (dashed border) tiles leads to three non-overlapping
    tiles. Data as in Figure~\ref{fig:permutations}.}
  \label{fig:merging}
\end{figure}

We use the term \emph{tiling} for a set of tiles $\mathcal{T}$ where
no tiles overlap.  Next, we show that there always exists a tiling
equivalent to a set of tiles.
\begin{theorem}
Given a set of (possibly overlapping) tiles $T$, there exists a tiling
 $\mathcal{T}$ that is equivalent to $T$.
\end{theorem}
\begin{proof}
Let $t_1 = (R_1, C_1)$ and $t_2 = (R_2, C_2)$ be two overlapping
tiles.  Each tile describes a set of constraints on the allowed
permutations of the rows in their respective column sets $C_1$ and
$C_2$. A tiling $\{t'_1, t'_2, t'_3\}$ equivalent to $\{t_1,t_2\}$ is
given by:
$$\begin{array}{rcl}
  t'_1& = &(R_1\setminus R_2, C_1), \\
  t'_2 &= &(R_1\cap R_2, C_1\cup C_2), \\
  t'_3 &= &(R_2\setminus R_1, C_2).
\end{array}$$
\noindent Tiles $t'_1$ and $t'_3$
represent the non-overlapping parts of $t_1$ and $t_2$ and the
permutation constraints by these parts can be directly met. Tile
$t'_2$ takes into account the combined effect of $t_1$ and $t_2$ on
their intersecting row set, in which case the same permutation
constraints must apply to the union of their column sets. It follows
that these three tiles are non-overlapping and enforce the combined
constraints of tiles $t_1$ and $t_2$. Hence, a tiling can be
constructed by iteratively resolving overlap in a set of tiles until
no tiles overlap.
\end{proof}
Notice that merging overlapping tiles leads to wider (larger column
set) and lower (smaller row set) tiles. An example is shown in
Figure~\ref{fig:merging}. The limiting case is a fully-constrained
situation where each row is a separate tile and only the identity
permutation is allowed. We provide an efficient algorithm with the
time complexity $\mathcal{O}{(n m)}$ for merging tiles in
Appendix~\ref{app:algo}.

\subsection{Formulating Hypotheses}
\label{sec:formulatinghypotheses}
As discussed in the introduction, in order to model what the user
wants to know about the data, we define two sets of constraints:
${\cal{T}}_H$ (the relations that are of the interest for the user)
and ${\cal{T}}_{H'}$ (the relations that are of no interest for the
user).  We then find a view that shows the maximal difference between
the uniform distributions from
$\Omega_{{\cal{T}}_u\cup{\cal{T}}_{H'}}$ and
$\Omega_{{\cal{T}}_u\cup{\cal{T}}_{H'}\cup {\cal{T}}_H}$,
respectively.  We now formalise this idea by formulating a pair of
\emph{hypotheses}, concisely represented using the tilings defined in
the previous section.  This provides a flexible method for the user to
specify the relations in which he or she is interested.

\begin{definition}[Hypothesis tilings]
  \label{def:hypothesis}
Given a subset of rows $R\subseteq[n]$, a subset of columns
$C\subseteq[m]$, and a $k$-partition of the columns given by
$C_1,\ldots,C_k$, such that $C=\cup_{i=1}^k{C_k}$ and $C_i\cap
C_j=\emptyset$ if $i\ne j$, a pair of hypothesis tilings is given by
$\mathcal{T}_{H_1}=\{(R,C)\}$ and
$\mathcal{T}_{H_2}=\cup_{i=1}^k{\{(R,C_i)\}}$.
\end{definition}
The hypothesis tilings define the items $R$ and attributes $C$ of
interest, and, through the partition of $C$, the relations between the
attributes in which the user is interested. \textsc{Hypothesis~1}
($\mathcal{T}_{H_1}$) corresponds to a uniform distribution from the
data sets in which all relations in $(R,C)$ are preserved, and
\textsc{hypothesis 2} ($\mathcal{T}_{H_2}$) to a uniform distribution
from the data sets in which the relations between attributes in the
partitions $C_1,\ldots,C_k$ of $C$ are broken while the relations
between attributes inside each $C_i$ are preserved.  In terms of the
constraints, as discussed in Section~\ref{sec:intro}, relations
preserved by $\mathcal{T}_{H_2}$ correspond to the relations which the
user is not interested in (that is, ${\cal{T}}_{H'}$ in
Section~\ref{sec:intro}), while relations preserved by
$\mathcal{T}_{H_1}$ but broken by $\mathcal{T}_{H_2}$ correspond to
relations which user is interested in (that is, ${\cal{T}}_{H}$ in
Section~\ref{sec:intro}).

\begin{figure}[t]
  \centering
  \begin{subfigure}{\textwidth}
    \begin{minipage}[t]{.46\linewidth}
      \vspace{0pt} \textbf{Case 1:}
      Are there relations between any of the attributes?
    \end{minipage}
    \hspace{2em}
    \begin{minipage}[t]{.43\linewidth}
      \vspace{0pt}
      \includegraphics[width = \textwidth]{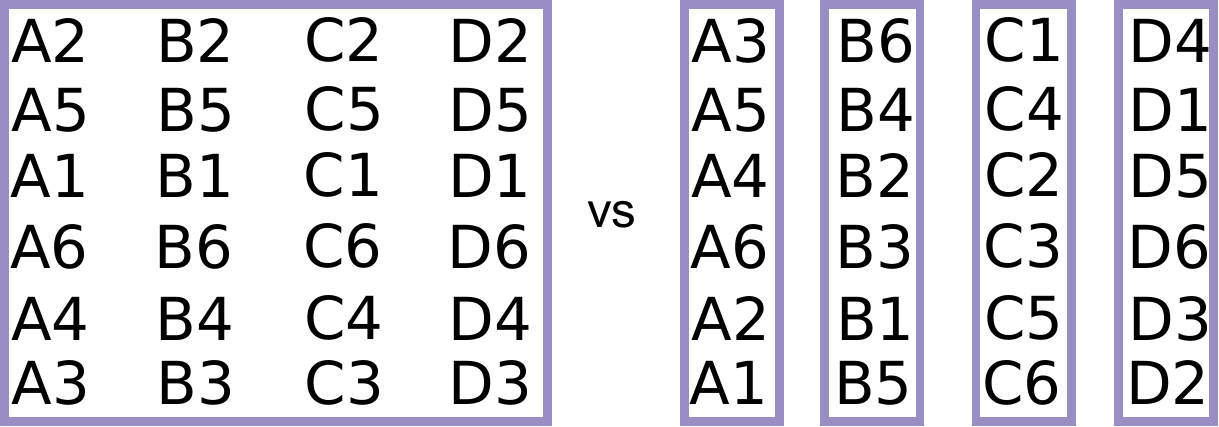}
    \end{minipage}
  \end{subfigure}

  \vspace*{2em}

    \begin{subfigure}{\textwidth}
    \begin{minipage}[t]{.46\linewidth}
      \vspace{0pt} \textbf{Case 2:}
            Are there relations between any of the attributes other than those
            constrained by the tile shown with an orange dashed border?
    \end{minipage}
    \hspace{2em}
    \begin{minipage}[t]{.43\linewidth}
      \vspace{0pt}
      \includegraphics[width = \textwidth]{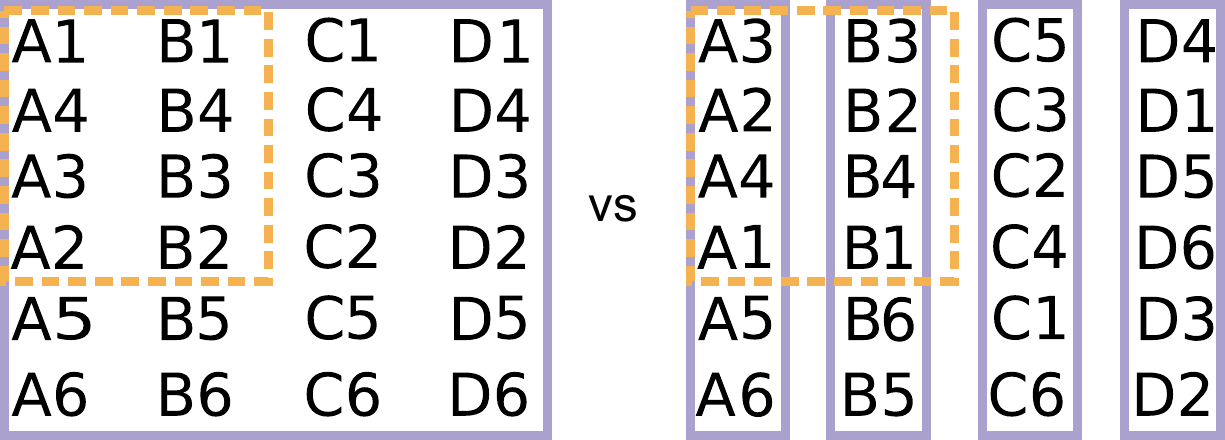}
    \end{minipage}
  \end{subfigure}

  \vspace*{2em}
  \begin{subfigure}{\textwidth}
    \begin{minipage}[t]{.46\linewidth}
      \vspace{0pt} \textbf{Case 3:} Are there relations between the
      attribute groups $\{A \}$ and $\{ B,C \}$ for items $\{ 3,4,5,6
      \}$ other than those constrained by the tile with an orange
      dashed border?
          \end{minipage}
  \hspace{2em}
    \begin{minipage}[t]{.43\linewidth}
      \vspace{0pt}
      \includegraphics[width = \textwidth]{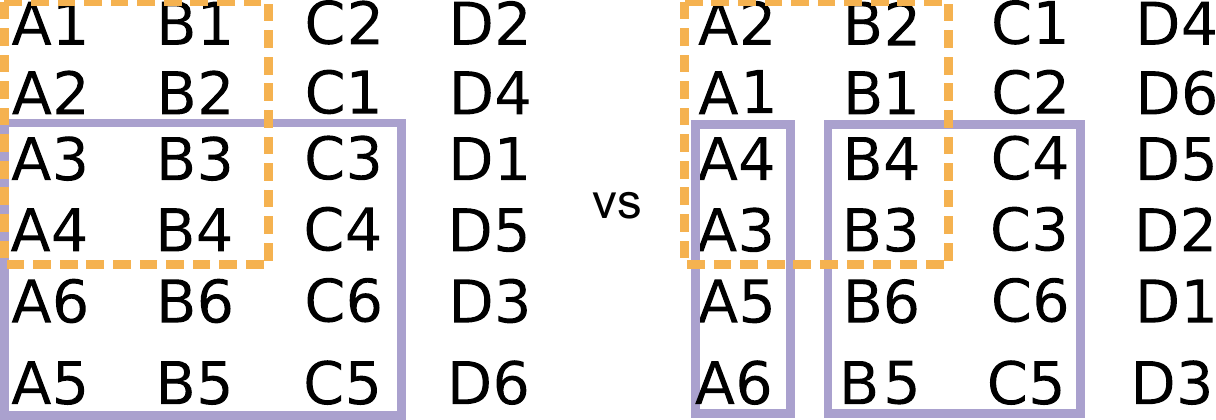}
    \end{minipage}
  \end{subfigure}

  \caption{Modelling hypotheses using sets of tiles. Tilings
    corresponding to hypotheses are shown with solid purple borders:
    \textsc{Hypothesis 1} on the left and \textsc{Hypothesis 2} on the
    right. The user's knowledge is represented by the tile with a
    dashed orange border. The data set has four attributes $A, B, C$
	and $D$, and six data items. See Section \ref{sec:iris} for
	an interpretation of these operations in terms of the iris flower data.}
  \label{fig:hypothesestiles}
\end{figure}

Now, for example, if the columns are partitioned into two groups $C_1$
and $C_2$ the user is interested in relations \emph{between} the
attributes in $C_1$ and $C_2$, but not in relations \emph{within}
$C_1$ or $C_2$. On the other hand, if the partition is full, that is,
$k=|C|$ and $|C_i|=1$ for all $i\in [k]$, then the user is interested
in \emph{all} relations between the attributes inside $(R,C)$. The
special case of $R=[n]$ and $C=[m]$ indeed reduces to {\em unguided
  data exploration}, where all inter-attribute relations in the data
are of interest to the user.

Having defined both the user's knowledge (background distribution) and
the pair of hypotheses formalising the user's objectives with tile
constraints, we can now easily combine these to formalise the uniform
distributions from the data sets we want to compare. I.e., we want to
compare the uniform distributions from $\Omega_{\mathcal{T}_u+
  \mathcal{T}_{H_1} }$ and $\Omega_{\mathcal{T}_u+
  \mathcal{T}_{H_2}}$, where `$+$' is used with a slight abuse of
notation to denote the operation of merging two tilings (with possible
overlaps between their tiles) into an equivalent tiling.  Notice here
that, by Definition \ref{def:hypothesis}, it holds that
\begin{displaymath}
\Omega_{\mathcal{T}_u+ \mathcal{T}_{H_1} } \subseteq \Omega_{\mathcal{T}_u+ \mathcal{T}_{H_2}},
\end{displaymath}
and hence the comparison becomes equivalent to the formulation
provided in the introduction.  Recall that we can draw samples from
these distributions as described in Section~\ref{sec:bg}.  From now
on, we use the term \emph{hypothesis pair}~$\mathcal{H}$ to denote
\begin{equation*}
  \mathcal{H} =\langle \mathcal{T}_u + \mathcal{T}_{H_1}, \mathcal{T}_u + \mathcal{T}_{H_2}\rangle,
\end{equation*}
where $ \mathcal{T}_u$ is the tiling formalising the (current)
background distribution, and $\mathcal{T}_{H_1}$ and
$\mathcal{T}_{H_2}$ form a pair of hypothesis tilings as defined in
Definition \ref{def:hypothesis}.  See Figure~\ref{fig:hypothesestiles}
for examples of different cases in which a pair of hypothesis tilings
and the user's knowledge are used to explore relations between
attributes. Here, the first case demonstrates a scenario in which the
user is interested in all relations, and hence the set of relations
that are of no interest to the user ${\cal{T}}_{H'}$ is
empty. Furthermore, the user has no prior knowledge, that is,
${\cal{T}}_{u}=\emptyset$.  The second case is similar, but here the
tile shown with an orange dashed border represents the user's
knowledge ${\cal{T}}_{u}\ne\emptyset$. Finally, the third case shows a
scenario in which ${\cal{T}}_{H'}\ne\emptyset$. The relations
preserved between attributes $B$ and $C$ for items $\{3,4,5,6\}$ are
of no interest to the user, while the relations between the attribute
groups $\{A\}$ and $\{B,C\}$ are.

\subsection{Finding Views}
\label{sec:view}
We are now ready to formulate our second main problem, that is, given
the uniform distributions from the two data sets characterised by the
hypothesis pair $\mathcal{H}$, how can we find an \emph{informative
  view} of the data maximally contrasting these? The answer to this
question depends both on the type of data and the selected
visualisation. For example, visualisations or measures of difference
are different for categorical and real-valued data. The
\emph{real-valued data} discussed in this paper allows us to use
projections (such as principal components) that mix attributes.
\begin{problem}[Comparing hypotheses]
\label{prob:comparinghypotheses}
Given the  two uniform distributions characterised by the hypothesis pair
$\mathcal{H} =\langle \mathcal{T}_u + \mathcal{T}_{H_1}, \mathcal{T}_u
+ \mathcal{T}_{H_2}\rangle$,
find the projection in which the distributions differ the
most.
\end{problem}

To solve this problem, we devise a \emph{linear projection pursuit
  method} which finds the direction in which the two distributions
differ the most in terms of \emph{variance}. In principle some other
difference measure could be used instead. However, a variance-based
measure can be implemented efficiently, which is one essential
requirement for interactive use.  Furthermore, using variance leads to
the convenient property that our projection pursuit method reduces to
standard principal component analysis (PCA) when the user has no
background knowledge and when the hypotheses are most general, as
shown in Theorem~\ref{thm:pca} below.

Thus, we formalise the optimisation criterion in
Problem~\ref{prob:comparinghypotheses} by defining a measure using
variance. Specifically, we choose the following form for our
\emph{gain function}:
\begin{equation}\label{eq:g}G(v,\mathcal{H})=\frac{v^T\Sigma_1v}{v^T\Sigma_2v},\end{equation}
where $v$ is a vector in $\mathbb{R}^m$ and $\Sigma_1$ and $\Sigma_2$
are the covariance matrices of the uniform distributions from
the data sets in $\Omega_{\mathcal{T}_u + \mathcal{T}_{H_1}}$ and
 $\Omega_{\mathcal{T}_u
+ \mathcal{T}_{H_2}}$, respectively. Then, the direction in which the
distributions differ most  in terms of the variance, that is, the solution to
Problem~\ref{prob:comparinghypotheses}, is given by
\begin{equation}
\label{eq:gain}
v_{\mathcal{H}}=\argmax_{v \in \mathbb{R}^m} G(v,\mathcal{H}).
\end{equation}

In order to solve Problem~\ref{prob:comparinghypotheses}, we first
show that the covariance matrix $\Sigma$ for a distribution defined
using the permutation-based scheme with tile constraints can be
computed analytically.
\begin{theorem}
Given
$j,j'\in[m]$, the covariance of attributes $\mathrm{cov}(j,j')$
from the uniform
distribution of data sets $\Omega_\mathcal{T}$ defined by a tiling $\mathcal{T}$
is given by
$\mathrm{cov}(j,j')=\sum\nolimits_{i=1}^n{a_j(i)a_{j'}(i)}/n$, where
\begin{equation*}
a_l(i)=\left\{
\begin{array}{ll}
Y(i,l),                                      &  i \in  R_{j,j'} \\
\sum\nolimits_{k\in R(i,l)}{Y(k,l)}/|R(i,l)|,  & i \notin  R_{j,j'}
\end{array}
\right.
\end{equation*}
and $l\in\{j,j'\}$. Here, $R_{j,j'}=\{i\in[n] \mid \exists (R,C) \in
\mathcal{T} \textrm{ where } i\in R \textrm{ and } j,j'\in C\}$
denotes the set of rows permuted together,
$Y(i,l)=X(i,l)-\sum\nolimits_{i=1}^n{X(i,l)}/n$ denotes the centred
data matrix, and $R(i,l)\subseteq[n]$ denotes a set satisfying
$\exists C\subseteq[m]$ where $(R(i,l),C)\in\mathcal{T}$, $i\in
R(i,l)$ and $l\in C$, that is, the rows in a tile that the data point
$X(i,l)$ belongs to.
\end{theorem}
\begin{proof}
  The covariance is defined by
  $$\mathrm{cov}(j,j')=
  E\left[\sum\nolimits_{i=1}^n{Y(\pi_j(i),j)Y(\pi_{j'}(i),j')}/n \right],$$
  where the expectation $E[\cdot]$ is defined over the
  permutations $\pi_j\in{\cal P}^n$ and $\pi_{j'}\in{\cal P}^n$ of columns
  $j$ and $j'$ allowed by the tiling $\mathcal{T}$, respectively. The part of the sum for rows permuted
  together $R_{j,j'}$ reads
  \begin{equation*}
  \sum\nolimits_{i\in R_{j,j'}}{E\left[Y(\pi_j(i),j)Y(\pi_{j'}(i),j')\right]}/n =
  \sum\nolimits_{i\in R_{j,j'}}{Y(i,j)Y(i,j')}/n,
  \end{equation*}
where we have used $\pi_j(i)=\pi_{j'}(i)$ and reordered the sum for
$i\in R_{j,j'}$. The remainder of the sum reads
   \begin{equation*}\sum\nolimits_{i\in
    R_{j,j'}^c}{E\left[Y(\pi_j(i),j)Y(\pi_{j'}(i),j')\right]}/n=
\sum\nolimits_{i\in
  R_{j,j'}^c}{E\left[Y(\pi_j(i),j)\right]E\left[Y(\pi_{j'}(i),j')\right]}/n,
   \end{equation*}
where $R_{j,j'}^c=[n]\setminus R_{j,j'}$ and the expectations have
been taken independently, because the rows in $R_{j,j}^c$ are permuted
independently at random. The result then follows from the observation
that $E\left[Y(\pi_l(i),l)\right]=a_l(i)$ for any $i\in
R_{j,j'}^c$.\footnote{We have also verified experimentally that the
  analytically derived covariance matrix matches the covariance matrix
  estimated from a sample from the distribution.}
\end{proof}

The direction in which the ratio of the variances is the largest can
now be found by applying a whitening operation
\citep{doi:10.1080/00031305.2016.1277159} on $\Sigma_2$.  The idea of
whitening is to find a \emph{whitening matrix} $W$ such that
$W^T\Sigma_2W=I$.  Using this transformation in Equation~\eqref{eq:g}
makes the denominator constant, and we hence obtain the solution to
the optimisation in Equation~\eqref{eq:gain} by finding the principal
components of $\Sigma_1$ transformed using $W$.

\begin{theorem}
\label{thm:opt}
The solution to the optimisation problem of Equation~\eqref{eq:gain}
is given by $v_{\mathcal{H}}=Ww$, where $w$ is the first principal
component of $W^T\Sigma_1W$ and $W\in{\mathbb{R}}^{m\times m}$ is a
whitening matrix such that $W^T\Sigma_2W=I$.
\end{theorem}
\begin{proof}Using $v=Ww$ the gain in Equation~\eqref{eq:g} can be rewritten as
  \begin{align}
    G(Ww,\mathcal{H})&=\frac{w^TW^T\Sigma_1Ww}{w^TW^T\Sigma_2Ww}
    =\frac{w^TW^T\Sigma_1Ww}{w^Tw}.
     \label{eq:Wv}
  \end{align}
Equation \eqref{eq:Wv} is maximised when $w$ is the maximal variance
direction of $W^T\Sigma_1W$, from which it follows that the solution
to the optimisation problem of Equation \eqref{eq:gain} is given by
$v_{\mathcal{H}}=Ww$, where $w$ is the first principal component of
$W^T\Sigma_1W$.
\end{proof}
\paragraph{Note} In visualisations
(that is, when making two-dimensional scatterplots), we project the
data onto the {\em first two principal components}, instead of
considering only the first component as above. \\

Finally, we are ready to show that at the limit of no background
knowledge and with the most general hypotheses, our method reduces to
PCA of the correlation matrix.
\begin{theorem}
\label{thm:pca}
In the special case of the first step in unguided data exploration,
that is, comparing distributions from a hypotheses pair specified by
$\mathcal{H} =\langle \emptyset+\mathcal{T}_{H_1}, \emptyset+
\mathcal{T}_{H_2}\rangle$, where $\mathcal{T}_{H_1}=\{([n],[m])\}$ and
$\mathcal{T}_{H_2}=\cup_{j=1}^m\{([n],\{j\})\}$, the solution to
Equation~\eqref{eq:gain} is given by the first principal component of
the correlation matrix of the data, when the data attributes have been scaled to
unit variance.
\end{theorem}
\begin{proof}
The proof follows from the observations that for $\mathcal{T}_{H_2}$
the covariance matrix $\Sigma_2$ is a diagonal matrix (here a unit
matrix), resulting in the whitening matrix $W=I$. For this pair of
hypothesis, $\Sigma_1$ denotes the covariance matrix of the original
data. The result then follows from Theorem~\ref{thm:opt}.
\end{proof}

Once we have defined the most informative projection, which displays
the directions in which the distributions parametrised by the
hypothesis pair $\mathcal{H}$ differ the most, we can show the
original data in this projection. This allows the user to observe
different patterns, for example, a clustered set of points, a linear
relationship, or a set of outlier points.  We note that it would also
be possible to show and compare samples from the two distributions
characterised by the hypothesis pair $\mathcal{H}$ in the most
informative view.  In \citet{ecml-pkdd:tiler} we presented a
proof-of-concept tool using which the user can, in fact, toggle
between showing the data and samples from the two distributions
representing the hypotheses. This can potentially shed some further
light into why this particular view is interesting, but as we are
mostly interested in the relations present in the actual data, we have
chosen for simplicity to consider only the data in the most
informative projection in this work.

\subsection{Selecting Attributes for a Tile Constraint}
\label{sec:attributeselection}
After observing a pattern, the user can define a tile $(R,C)$ to be
added to $\mathcal{T}_u$. The set of data points $R$ included in the
pattern can be easily selected from the projection shown. For
selecting the attributes characterising the pattern, we use a
procedure where for each attribute the ratio between the standard
deviation of the attribute for the selection and the standard
deviation of all data points is computed. If this ratio is below a
threshold value $\tau$ (for example, $\tau=0.5$), then the attribute
is included in the set of attributes $C$ characterising the
pattern. The intuition here is that we are looking for attributes in
which the selection of points are more similar to each other than is
expected based on the whole data. Thus, the set of attributes for
which the user's knowledge of dependencies is included, is affected by
the choice of $\tau$. A smaller value of $\tau$ will only include
attributes for which the selection of points is very similar, whereas
a larger value of $\tau$ will include a larger set of attributes to
the tile constraint. A parallel coordinates plot ordered according to
the ratio of standard deviations can be useful in deciding a suitable
value for $\tau$, see Section \ref{ssec:german} for examples in which
we use parallel coordinate plots.

\begin{figure}
\begin{tabular}{c@{}c@{}c}
\includegraphics[width=0.325\textwidth, trim = 8mm 8mm 8mm 8mm, clip]{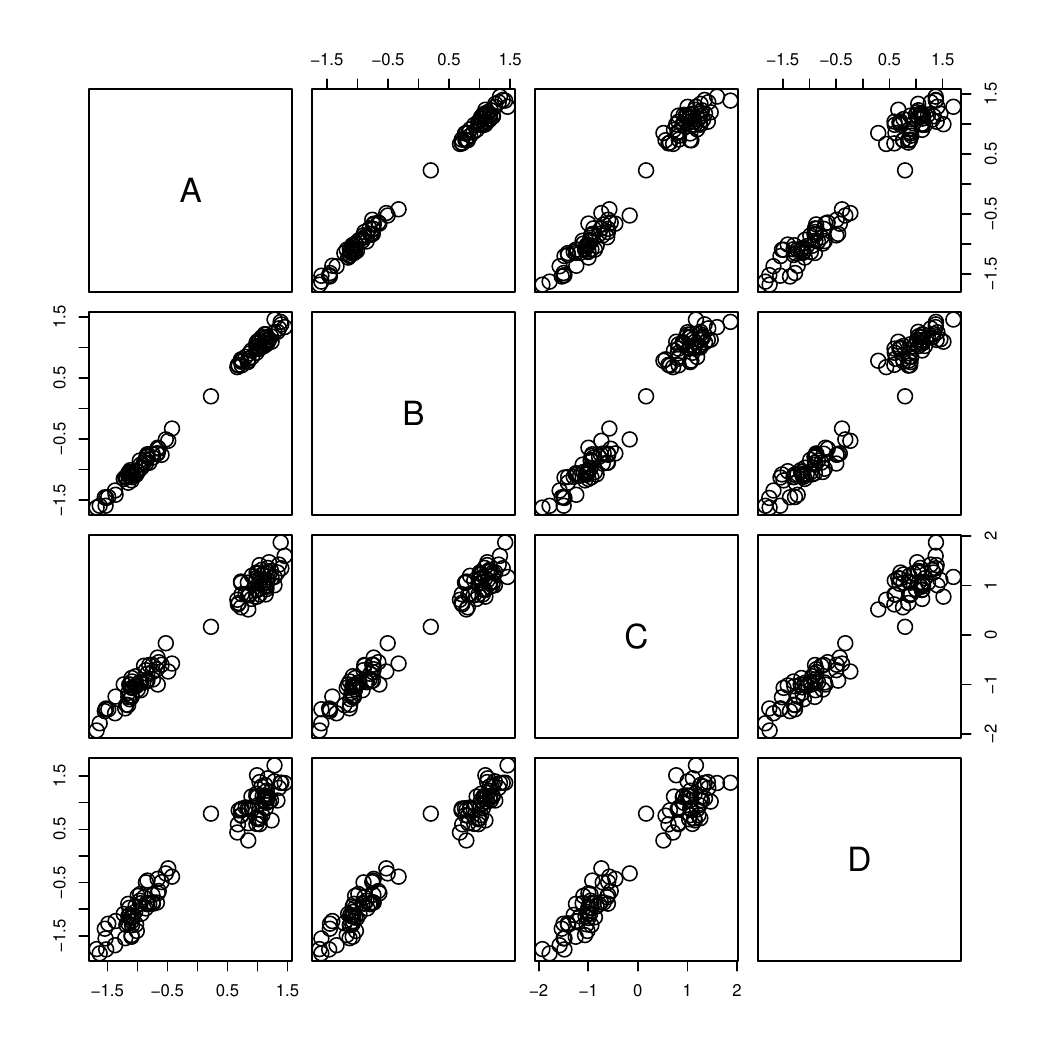} &
\includegraphics[width=0.325\textwidth, trim = 8mm 8mm 8mm 8mm, clip]{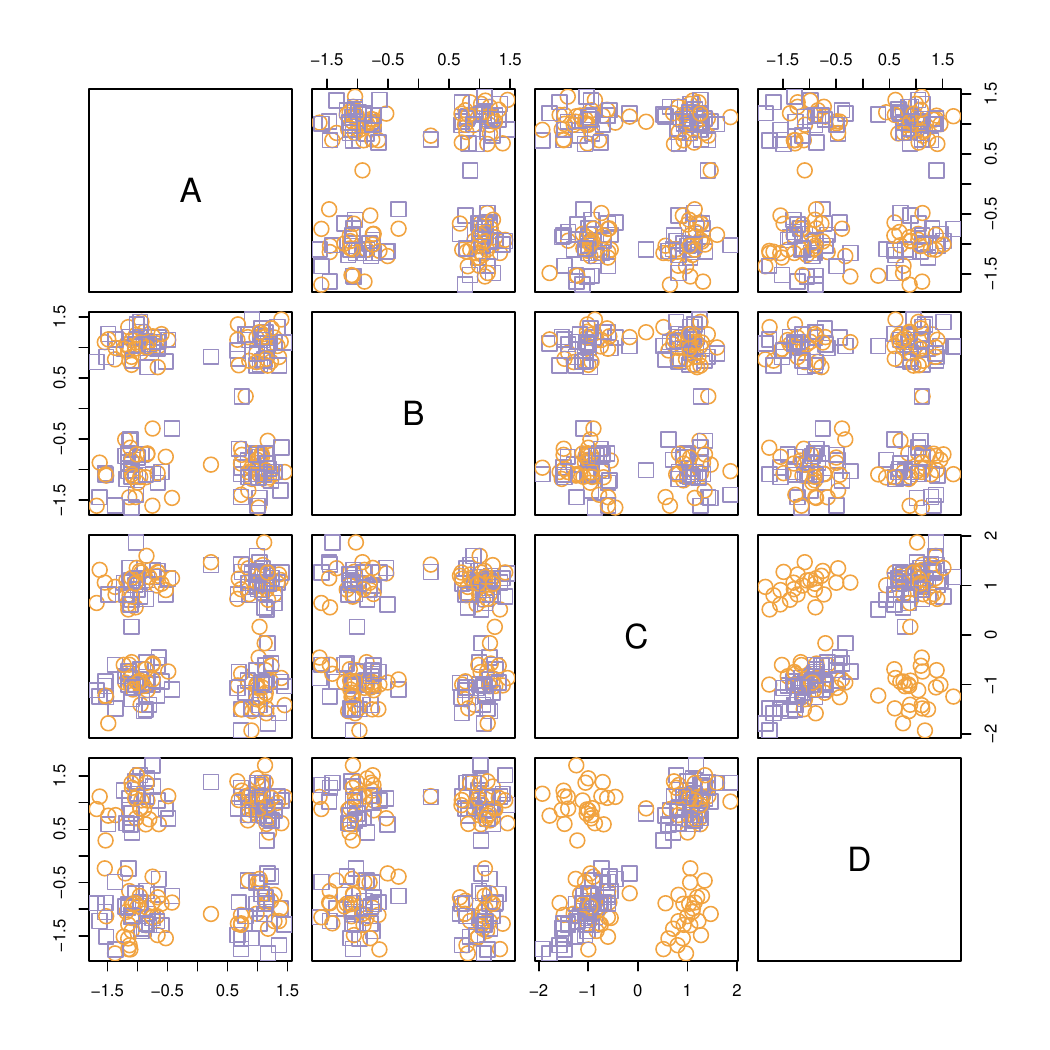} &
\includegraphics[width=0.325\textwidth, trim = 8mm 8mm 8mm 8mm, clip]{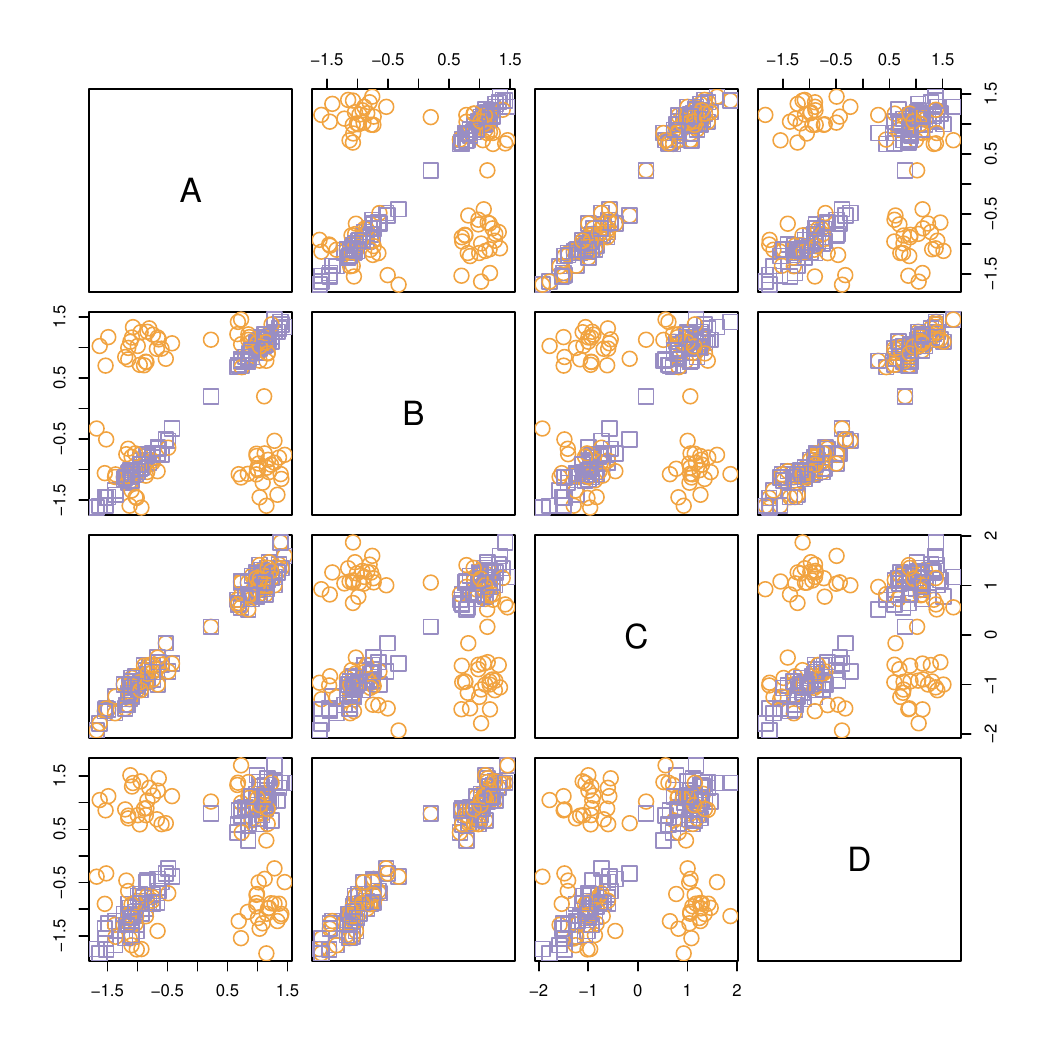} \\
 (a) & (b) & (c)\\
\end{tabular}
\caption{\label{fig:toy} (a) Scatterplot matrix of the toy data
  set. (b) Fully randomised data (orange circles) modelling the user's
  knowledge and data (purple squares) where only the relation between
  attributes $C$ and $D$ has been preserved modelling what the user
  could learn of the relations between attributes $C$ and $D$. (c) As
  in (b), but additionally the relation between attributes $A$ and $C$, as
  well as between attributes $B$ and $D$, has been preserved,
  modelling the user's knowledge of the relations between $A$ and $C$ as
  well as $B$ and $D$, respectively.}
\end{figure}

\subsection{Example: Iris Data}\label{sec:iris}

The {\sc iris} flower dataset is a well-known machine learning dataset. It consists of  four measurements on 150 iris flowers from three species \citep{iris}. In this section we discuss, for illustration, what the operations shown in Figure \ref{fig:hypothesestiles} could mean in terms of the {\sc iris} data.

Assume that the four attributes $A$--$D$ in Figure \ref{fig:hypothesestiles} are the four measurements: {\em petal length} ($A$), {\em petal width} ($B$), {\em sepal length} ($C$), and {\em sepal width} ($D$). Further assume that rows 1--2 correspond to the species {\em setosa}, rows 3--4 to the species {\em versicolor}, and rows 5--6 to the species {\em virginica},\footnote{Notice that we actually have 50 (not 2) specimen per species! We compute the projections on the full {\sc iris} data.} and that the attributes have been scaled to zero mean and unit variance.

Case 1 of Figure \ref{fig:hypothesestiles} corresponds to the PCA projection of the {\sc iris} data, as stated by Theorem \ref{thm:pca}. The first PCA component, given by Equation \eqref{eq:gain}, is $v_{\mathcal{H}}=-0.58\times A-0.56\times B-0.52\times C+0.26\times D$, implying that the overall correlations of the data are best described by (roughly) the average petal length ($A$), petal width ($B$), and sepal length ($C$).

In Case 2 we assume that the user is aware of the high correlation of petal length ($A$) and petal width ($B$) for the species setosa and versicolor, which is expressed by the tile with the orange dashed border in Figure \ref{fig:hypothesestiles}. With this constraint, the direction provided by Equation \eqref{eq:gain} is given by $v_{\mathcal{H}}=-0.73\times A+0.11\times B-0.58\times C+0.34\times D$. The weight of the petal width is reduced---the user already knows about its correlation with petal length---and the weight of the other measurements is increased.

In Case 3 the user is interested only in the species versicolor and virginica and the correlation between petal length  ($A$) and petal width and sepal length ($B$--$C$). In this case the most informative direction given by Equation \eqref{eq:gain} is  $v_{\mathcal{H}}=0.50\times A-0.77\times B+0.39\times C+0\times D$. Sepal width ($D$) is no longer informative, because it is not connected to the hypothesis pair $\mathcal{H}$, and it appears in the projection with a zero weight.

Summarising, the projection therefore tends to emphasize more the directions that are most informative to the user, of which we show more examples in Section \ref{sec:experiments}.

\subsection{Example: Analysing Partitions of Data Items}

Partitioning of data items or subsets of data items can be found, e.g., by clustering algorithms, by using known categorical attributes, or by visual inspection, as we will later demonstrate in Sections \ref{ssec:german} and \ref{ssec:other}. Our approach is suitable for probing such subsets of data items and finding relations in data that are not explained by the given data subset.

At the simplest, if we are given a subset of data items of interest, we can use the hypothesis tiling of Definition \ref{def:hypothesis} to probe the subset of rows $R$ of interest. Another 
straightforward to include a subset of data items $R$ into the user's background information would be to add a tile $(R,[m])$ to $\mathcal{T}_u$, after which the system would ignore the data items in $R$. More nuanced examples of adding subsets of data points into background information are given in Section \ref{sec:experiments}.

\subsection{Example: Subsetting Loses Information}
We conclude this section with a simple example illustrating how
subsetting the data in order to focus on a specific objective can lead
to loss of information. With this example we wish to highlight two
aspects: (i) how the user's background knowledge and objectives affect
the views that are most informative, and (ii) how it can be
advantageous to investigate relations in the data as a whole instead
of using a subset of the data for the analysis.

We construct a toy data set with four attributes $A$, $B$, $C$, and
$D$ as follows. We first generate two strongly correlated attributes
$A$ and $B$, after which we generate attribute $C$ by adding noise to
$A$, and attribute $D$ by adding noise to $B$. This data set,
visualised in Figure~\ref{fig:toy}(a), is very simple and here it is
possible to investigate all pairwise relations in the data in one
view. This is in general not possible in any real analysis
scenarios. Furthermore, we assume that the user is interested in the
relation between attributes $C$ and $D$. Our goal is then to find a
maximally informative 1-dimensional projection of the data that takes
both this objective and the user's background knowledge into account.

First, let us assume that the user only knows the marginal
distribution of each attribute but is unaware of the relations between
the attributes. Using the approach in this paper we formulate this by
means of the hypothesis pair $\mathcal{H}_0 =\langle \mathcal{T}_{u_0}
+ \mathcal{T}_{H_1}, \mathcal{T}_{u_0} + \mathcal{T}_{H_2}\rangle$,
where $\mathcal{T}_{u_0}=\emptyset$,
$\mathcal{T}_{H_1}=\{(R,\{C,D\})\}$, $\mathcal{T}_{H_2}=\{(R,\{C\}),
(R,\{D\}) \}$, and $R = [n]$ (all rows in the data). A sample from
$\Omega_{\mathcal{T}_{u_0} + \mathcal{T}_{H_1}}$ is shown in
Figure~\ref{fig:toy}(b) using purple squares, and a sample from
$\Omega_{\mathcal{T}_{u_0} + \mathcal{T}_{H_2}}$ is shown using orange
circles. The orange distribution hence models what the user currently
knows and the purple what the user could optimally learn about the
relation between $C$ and $D$ from the data. The orange and purple
distributions differ the most in the plot $CD$, as expected, and
indeed the maximally informative 1-dimensional projection satisfying
Equation~\eqref{eq:gain}, is given by $v=0.7C+0.7D$.

Secondly, assume that, unlike above, the user is already aware of the
relationship between the attribute pairs $A$ and $C$ as well as $B$
and $D$, but does not know that attributes $A$ and $B$ are almost
identical. We proceed as above with the difference that we now add the
user's knowledge as a constraint to both the distributions.  This is
achieved by updating the hypothesis pair to $\mathcal{H}_1 =\langle
\mathcal{T}_{u_1} + \mathcal{T}_{H_1}, \mathcal{T}_{u_1} +
\mathcal{T}_{H_2}\rangle$, where $\mathcal{T}_{u_1}=\{(R,\{A,C\}),
(R,\{B,D\}) \}$ captures the user's knowledge.

Samples from the uniform distributions on the data sets conforming to
this hypothesis pair are shown in Figure~\ref{fig:toy}(c). Again, the
orange distribution models the user's knowledge (that is,
$\Omega_{\mathcal{T}_{u_1} + \mathcal{T}_{H_2}}$) and the purple what
the user could learn from the relation between $C$ and $D$ from the
data, given that the user already knows about the relationships of the
attribute pairs $A$ and $C$ as well as $B$ and $D$ (that is,
$\Omega_{\mathcal{T}_{u_1} + \mathcal{T}_{H_1}}$).  The orange and
purple distributions differ the most in the plot $AB$ and therefore
the user would gain most information if shown this view.  Indeed, the
most informative 1-dimensional projection satisfying
Equation~\eqref{eq:gain} is $v=-0.7A-0.7B$.  In other words, the
knowledge of the relation of $A$ and $B$ gives maximal information
about the relation of $C$ and $D$. This makes sense, because the
variables $C$ and $D$ are really connected via $A$ and $B$ through
their generative process.

This example hence shows how the background knowledge affects the
views. Also, if we had chosen a subset of the data containing, for
example, just attributes $C$ and $D$ we would not have observed the
connection of $C$ and $D$ through $A$ and $B$, even if we knew the
relation between $A$ and $C$ as well as $B$ and $D$. Thus, we have
demonstrated with this simple example that using hypothesis tilings as
above allows us to explore the entire data set at once while still
focusing on particular relations of interest.

% ------------------------------------------------------------
\section{Experiments}
\label{sec:experiments}
% ------------------------------------------------------------
In this section we first consider the stability and scalability of the
framework presented in this paper. After this, we present examples of
how the proposed method is used to explore relations in a data set and
to focus on investigating a hypothesis concerning relations in a
subset of the data. An open source library implementing the proposed
framework, including the code for the experiments presented in this
paper, is available from \url{https://github.com/edahelsinki/corand/}.

All the experiments were run on a MacBook Pro laptop with a
\unit[3.1]{GHz} Intel Core i5 processor using R version 3.5.2
\citep{Rproject}.

\subsection{Data Sets}
We use synthetic data in the scalability
experiment. We also use two real-world data sets to showcase the
applicability of our framework in human-guided data exploration.

The \textsc{german} socioeconomic data set \citep{boley2013,
  kang2016}\footnote{Available from
  \url{http://users.ugent.be/~bkang/software/sica/sica.zip}} contains
records from 412 German administrative districts. Each district is
represented by 46 attributes describing socioeconomic and political
aspects in addition to attributes such as the type of the district
(rural/urban), area name/code, state, region (East/West/North/South)
and the geographic coordinates of each district center. The {\em
  socioecologic attributes} include, for example, population density,
age and education structure, economic indicators (for example, GDP
growth, unemployment, income), and the proportion of the workforce in
different sectors. The {\em political attributes} include election
results of the five major political parties (CDU/CSU, SPD, FDP, Green,
and Left) in the German federal elections in 2005 and 2009, as well as
the voter turnout. For our experiments we exclude the election results
from 2005 (which are highly correlated with the 2009 election
results), all non-numeric variables, and the area code and coordinates
of the districts, resulting in 32 real-valued attributes (although we
use the full data set when interpreting the results). Finally, we
scale the real-valued variables to zero mean and unit variance.

The \textsc{accident} data set\footnote{Proprietary data obtained from
  the Finnish \emph{Workers' Compensation Center}
  \url{https://www.tvk.fi/}} is a random sample of 3000 accidents from
a large data set containing all occupational accidents in Finnish
enterprises during the period 2003--2014 reported to the Finnish
\emph{Workers' Compensation Center}. In the original data set, the
accidents are described by 37 variables, the majority of which are
categorical, including details about the victim (occupation, age, sex,
nationality) and the accident (geographical location, cause, type,
working process). We use \emph{one-hot encoding} to transform the categorical variables into
real-valued variables, creating a column
for every label of every variable in which the presence (absence) of a
label is indicated by 1 (0, respectively). To restrict the
dimensionality of the resulting encoding, we drop variables with a
very high number of labels; for example, the variable for the
municipality in which the accident happened has more than 300 labels
and would result in equally many columns in the data.  Variables with
many labels are implicitly given more weight in the one-hot encoding
as well. For instance, the attribute \textsc{sukup} (gender) has 2
labels, while the attribute \textsc{ruumis} (injured body part) has 68
labels. In the transformed data there are 2 columns for \textsc{sukup}
and 68 columns for \textsc{ruumis}, making the latter more strongly
represented in the data. This could impact further analysis, and to
overcome this effect, we scaled the binary data in groups, that is,
all columns that originate from the same variable are scaled to have a
total variance of 1. The resulting data set contains 3000 rows and 220
attributes.

\begin{table}[t]\centering
\centering
\begin{tabular}[t]{cccc}
\toprule
$\bm{\sigma}$ & $\bm{\Delta n=0}$ & $\bm{\Delta n=100}$ & $\bm{\Delta n=200}$\\
\midrule
$0$ & $0.000$  & $0.008$ &  $0.021$ \\
$1$ &  $0.049$ & $0.058$ &  $0.096$ \\
$2$ & $0.111$ & $0.144$ & $0.170$ \\
$5$ & $0.280$  &  $0.230$ &  $0.293$  \\
$10$ &  $0.358$ & $0.302$ & $0.308$ \\
\bottomrule
\end{tabular}
\caption{The mean relative error as a function of perturbance of the
  data. Here $\sigma$ is the variance of the noise added and $\Delta
  n$ denotes the number of rows randomly removed. The relative error is
  the difference in gain of Equation~\eqref{eq:g} between the
  optimal solution $v_{\mathcal{H}}$ and the solution
  $v_{\mathcal{H}}^*$ found on perturbed data divided by the gain in
  the optimal solution.}
\label{tab:r:a}
\end{table}

\begin{table}[t!]\centering
\centering
\begin{tabular}[t]{rrrr}
\toprule
$\bm{n}$ &
$\bm{m}$ &
$\bm{t_\mathrm{model}}$ \textbf{(s)} &
$\bm{t_\mathrm{view}}$ \textbf{(s)} \\
\midrule
$500$ & $10$ & $0.02$ & $0.01$ \\
$1000$ & $10$ & $0.04$ & $0.01$ \\
$5000$ & $10$ & $0.22$ & $0.04$ \\
$10000$ & $10$ & $0.53$ & $0.10$ \\\midrule
$500$ & $50$ & $0.06$ & $0.14$ \\
$1000$ & $50$ & $0.10$ & $0.20$ \\
$5000$ & $50$ & $0.41$ & $0.81$ \\
$10000$ & $50$ & $2.01$ & $1.65$ \\\midrule
$500$ & $100$ & $0.09$ & $0.48$ \\
$1000$ & $100$ & $0.14$ & $0.75$ \\
$5000$ & $100$ & $0.92$ & $3.32$ \\
$10000$ & $100$ & $3.40$ & $6.68$ \\\midrule
$500$ & $150$ & $0.13$ & $1.02$ \\
$1000$ & $150$ & $0.28$ & $1.65$ \\
$5000$ & $150$ & $1.58$ & $7.26$ \\
$10000$ & $150$ & $2.22$ & $15.20$ \\\midrule
$500$ & $200$ & $0.25$ & $1.74$ \\
$1000$ & $200$ & $0.40$ & $3.05$ \\
$5000$ & $200$ & $0.84$ & $13.08$ \\
$10000$ & $200$ & $6.67$ & $26.37$ \\
\bottomrule
\end{tabular}
\caption{Median wall clock running time for random data with varying
  number of rows ($n$) and columns ($m$) for a data set consisting of
  Gaussian random numbers. We give the time to add three random tiles
  plus hypothesis tiles ($t_\mathrm{model}$) and the time to find the
  most informative view ($t_\mathrm{view}$), that is, to solve
  Equation~\eqref{eq:gain}.}
\label{tab:r:b}
\end{table}

% ------------------------------------------------------------
\subsection{Stability and Scalability}
% ------------------------------------------------------------

We first study the sensitivity of the results with respect to noise or
missing data rows.  In this experiment we use the 32 real-valued
variables from the {\sc german} data together with three (non-trivial)
factors, namely {\em Type} (2 values), {\em State} (16 values), and
{\em Region} (4 values) to create synthetic data sets. A synthetic
data set, parametrised by the noise term $\sigma$ and an integer
$\Delta n$ is constructed as follows. First, we randomly remove
$\Delta n$ rows from the data, after which Gaussian noise with
variance $\sigma^2$ is added to the remaining variables, and finally
all variables are rescaled to zero mean and unit variance. We create a
random tile by randomly picking a factor that defines the rows in a
tile and then randomly sample 2 to 32 attributes as the columns.  The
background knowledge $\mathcal{T}_{u}$ consists of three such random
tiles.  The hypothesis tiles are constructed using one such random
tile $(R,C)$ as a basis: $\mathcal{T}_{H_1}=\{(R,C)\}$ and
$\mathcal{T}_{H_2}=\cup_{j\in C}{\{(R,\{j\})\}}$.

The results are shown in Table~\ref{tab:r:a}. We notice that the
method is relatively insensitive with respect to the gain in terms of
noise and removal of rows. Even removing about half of the rows
($\Delta n=200$) does not change the results meaningfully. Only a very
high degree of noise, corresponding to $\sigma\geq 5$ (that is, circa
5--10\% signal-to-noise ratio) substantially degrades the results.

Table~\ref{tab:r:b} shows the running time of the algorithm as a
function of the size of the data for Gaussian random data with a
similar tiling setup as used for the {\sc german} data. We make two
observations. First, the tile operations scale linearly with the size
of the data $nm$ and they are relatively fast. Most of the time is
spent on finding the views, that is, solving
Equation~\eqref{eq:gain}. Even our unoptimised pure R implementation
runs in seconds for data sets that are visualisable (having thousands
of rows and hundreds of attributes); any larger data set should in any
case be downsampled for visualisation purposes.

\begin{figure}[t]
\centering
\includegraphics[width=0.75\textwidth]{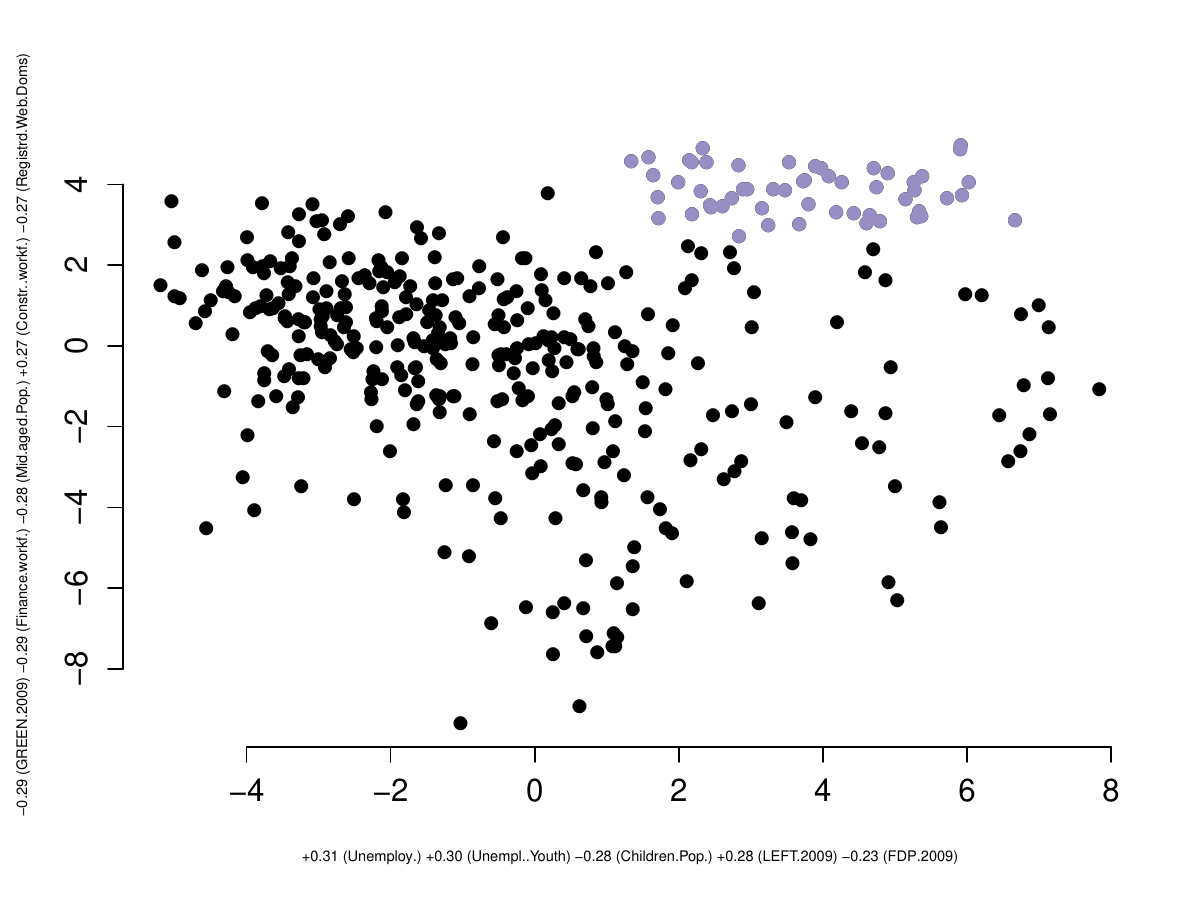}
\caption{The most informative view of the \textsc{german} data set
  with respect to the hypothesis pair $\mathcal{H}_{E,\emptyset}$.
  Black filled circles show the data points; the selected points ({\em
    Selection 1}) are marked with purple. The $x$ and $y$ axis labels
  show the five attributes with the largest absolute values in each
  projection vector.}
\label{fig:exploration1left}
 \end{figure}

\begin{figure}[t]
  \centering
  \includegraphics[width=0.62\columnwidth]{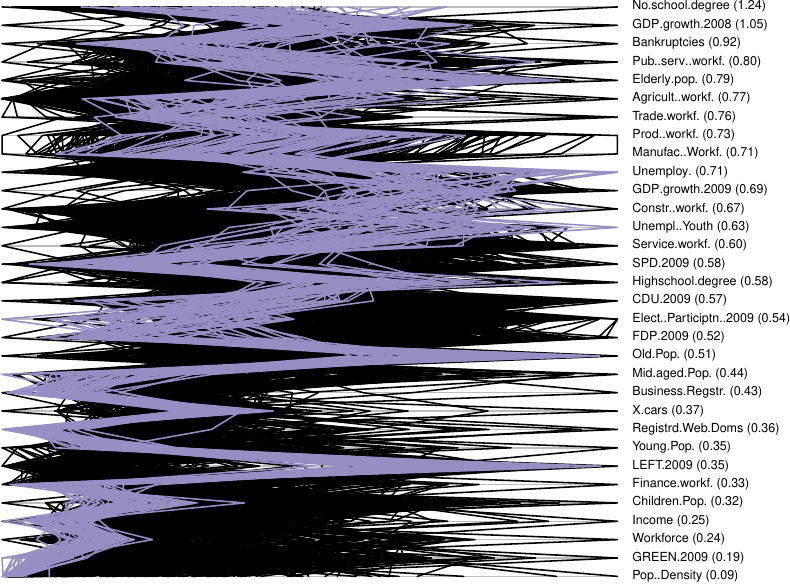}
  \caption{Parallel coordinates plot of the \textsc{german} data set
    with {\em Selection 1} of Figure  \ref{fig:exploration1left} highlighted in purple.
	Here each of the jigged vertical lines corresponds to one of the data items; the colours are as in Figure \ref{fig:exploration1left}. 
	Each of the horizontal lines corresponds to a data attribute. The values of the attributes have been shifted and scaled so that the minimum value of each attribute is at the left and the maximum value is at the right. The intersection of a horizontal line and a jigged vertical line (data item) corresponds to the value of the attribute for the particular data item. In the parenthesis we show the ratio between the standard deviation of the attribute for the selection (purple data items) and the standard deviation of the attribute for all data items, i.e., the parameter $\tau$ as defined Section \ref{sec:attributeselection}. The attributes have been ordered by decreasing $\tau$, after which the attributes that are most homogeneous for data items in this particular selection are at the bottom.}
    \label{fig:explore:pcp:c1}
 \end{figure}

% ------------------------------------------------------------
\subsection{Exploration of the German Data Set}
\label{ssec:german}
% ------------------------------------------------------------
Next, we demonstrate our  framework by
exploring the \textsc{german} data set under different objectives.

\paragraph{Exploration without prior background knowledge}
We start with {\em unguided data exploration} where we have no prior
knowledge about the data and our interest is as generic as possible.
In this case $\mathcal{T}_u=\emptyset$ and as the hypothesis tilings
we use $\mathcal{T}_{E_1}$, where all rows and columns belong to the
same tile (fully-constrained tiling), and $\mathcal{T}_{E_2}$, where
all columns form a tile of their own (fully unconstrained tiling). Our
hypothesis pair is then $\mathcal{H}_{E, \emptyset}=\langle\emptyset+
\mathcal{T}_{E_1}, \emptyset+\mathcal{T}_{E_2}\rangle$.

We then consider the view of the data
(Figure~\ref{fig:exploration1left}) which is maximally informative,
that is, in which the two distributions parametrised by the hypothesis
pair $\mathcal{H}_{E, \emptyset}$ differ the most. We observe that
there is some structure visible in this view. In order to investigate
the characteristics of the data points corresponding to different
patterns in the \textsc{german} data, we first choose to focus on the
set of points in the upper right corner, marked with purple in
Figure~\ref{fig:exploration1left}. Our selection, denoted by {\em
  Selection 1}, corresponds to rural districts in Eastern Germany (see
Table~\ref{tab:explore:clusters}). We also consider the parallel
coordinates plot of the data, shown in
Figure~\ref{fig:explore:pcp:c1}. This plot shows the 32 real-valued
attributes in the data. The currently selected points ({\em Selection
  1}) are shown in purple while the rest of the data is shown in
black. The number in parentheses following each variable name is the
ratio of the standard deviation of the selection and the standard
deviation of all data. If this number is small we can conclude that
the values for a particular attribute are homogeneous inside the
selection (behave similarly). Based on the parallel coordinates plot
in Figure~\ref{fig:explore:pcp:c1} we observe that there is little
support for the Green party and a high support for the Left party in
these districts.

\begin{table}[t]
\centering
  \begin{tabular}{lcccccc}
  \toprule
 &  \multicolumn{4}{c}{\textbf{Region}} &   \multicolumn{2}{c}{\textbf{Type}} \\
  \cmidrule(r){2-5} \cmidrule(l){6-7}
  \textbf{Selection} &  \textbf{North} & \textbf{South} & \textbf{West} & \textbf{East} & \textbf{Urban} & \textbf{Rural} \\
  \midrule
Selection 1 & 0 & 0 & 0 & 54 & 0 & 54 \\
Selection 2 & 10 & 7 & 21 & 22 & 60 & 0 \\
   \bottomrule
\end{tabular}
  \caption{Distribution of the {\em Region} and {\em Type} attributes
  for {\em Selection 1} and {\em Selection 2} in the \textsc{german} data
  set.}
  \label{tab:explore:clusters}
\end{table}

\begin{figure}[t]
\centering
\includegraphics[width=0.75\textwidth]{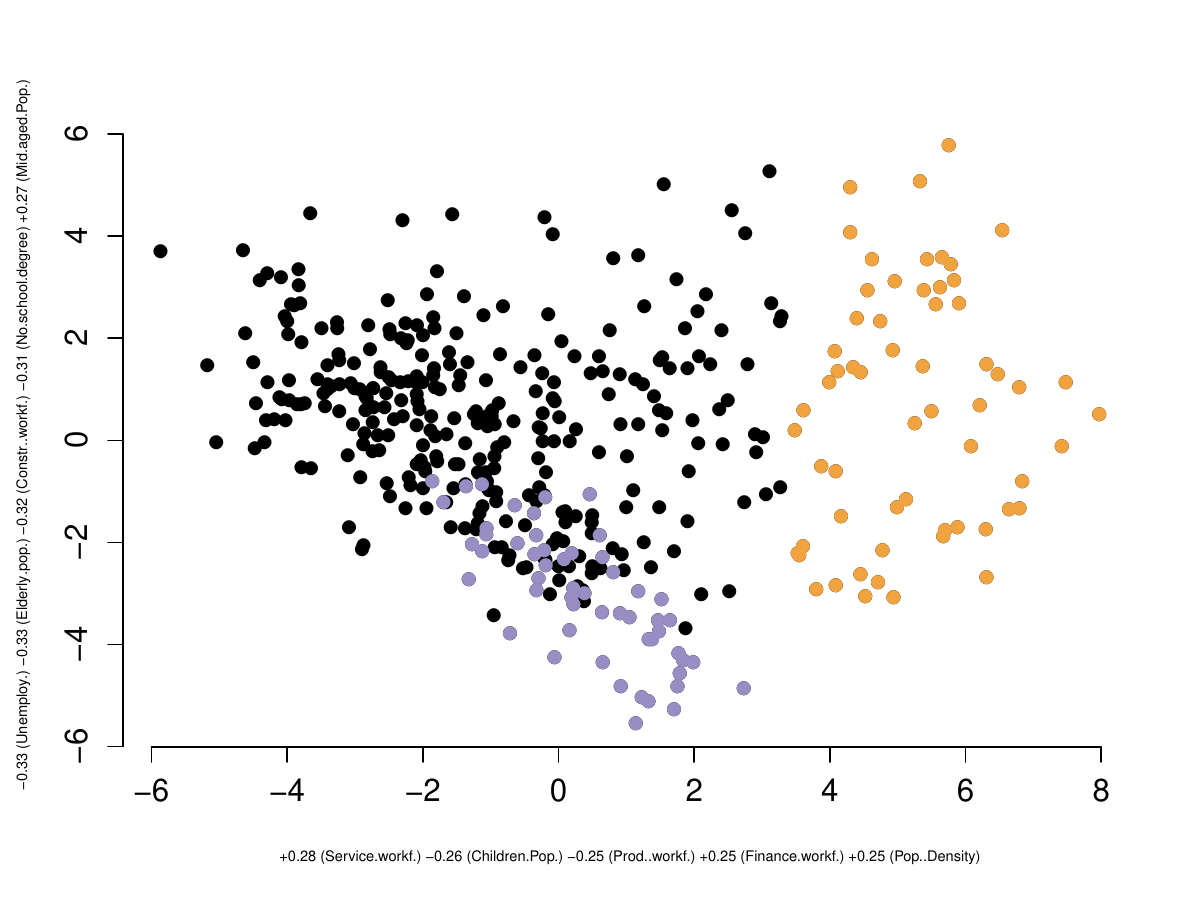}
\caption{The most informative view of the \textsc{german} data set
  with respect to $\mathcal{H}_{E,\{t\}}$. Black filled circles show
  the data points; the selected points, that is, {\em Selection 2} are
  shown in orange. The points corresponding to {\em Selection 1} from
  the previous exploration step are shown in purple. The $x$ and $y$
  axis labels show the five attributes with the largest absolute
  values in each projection vector.}
\label{fig:exploration1right}
 \end{figure}

  \begin{figure}[t]
  \centering
    \includegraphics[width=0.62\columnwidth]{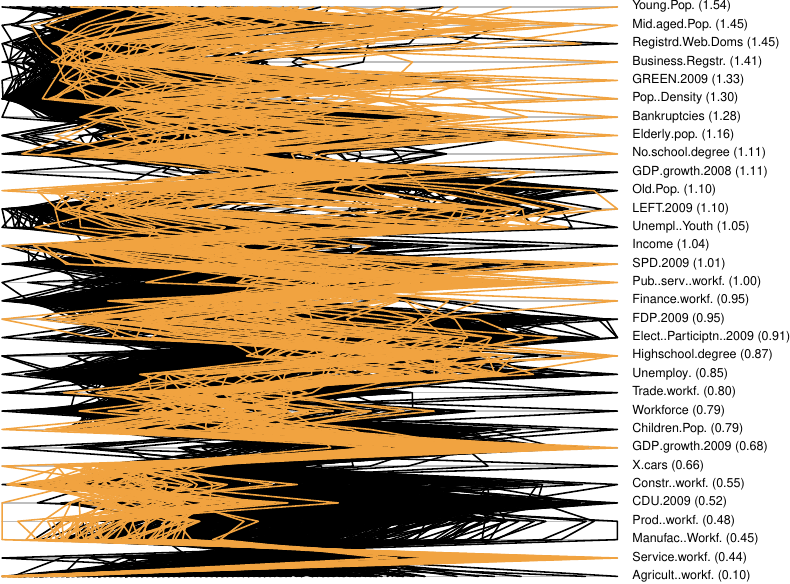}
  \caption{Parallel coordinates plot of the \textsc{german} data set
	  with {\em Selection 2} of Figure \ref{fig:exploration1right} highlighted in orange.
	  See the caption of Figure \ref{fig:explore:pcp:c1} for an explanation
	  of the semantics of the plot.}
    \label{fig:explore:pcp:c2}
 \end{figure}

We next add a tile constraint $t$ for the items in the observed
pattern where the columns (attributes) are chosen as described in
Section~\ref{sec:attributeselection} using a threshold value $\tau =
2/3$. Thus, we select those attributes for which the standard
deviation ratio, that is, the number in parentheses in
Figure~\ref{fig:explore:pcp:c1}, is below the threshold. The
hypothesis pair is then updated to take into account the newly added
tile, that is, we consider $\mathcal{H}_{E, \{t\}}=\langle \{t\}+
\mathcal{T}_{E_1}, \{t\}+\mathcal{T}_{E_2} \rangle$.

The most informative view displaying differences of the distributions
parametrised by $\mathcal{H}_{E, \{t\}}$ is shown in
Figure~\ref{fig:exploration1right}. Now, {\em Selection 1} (shown in
purple for illustration purposes) is no longer as clearly visible in
Figure~\ref{fig:exploration1right} as it is in the first view. This is
expected, since this pattern has been accounted for in the
distributions parametrised using $\mathcal{H}_{E, \{t\}}$.  We now
focus on investigating the sparse region of points shown in orange in
Figure~\ref{fig:exploration1right} (\emph{Selection 2}). By inspecting
the class attributes of this selection we learn that these items
correspond to urban districts (see Table~\ref{tab:explore:clusters})
in all regions. Based on the parallel coordinates plot shown in
Figure~\ref{fig:explore:pcp:c2} we conclude that these districts are
characterised by a low fraction of agricultural workforce and a high
amount of service workforce, both expected in urban districts. We also
notice that these districts have had a higher GDP growth in 2009 and
that it appears that the amount of votes for the CDU party in these
districts was quite low.

\paragraph{Exploration with a specific hypothesis}
Next, we focus on a more specific hypothesis involving only a subset
of rows and attributes. In particular, we want to investigate a
hypothesis concerning the relations between certain attribute groups
in {\em rural areas}. We hence define our hypothesis pair as
follows. As the subset of rows $R$ we choose {\em all 298 districts
  that are of the type rural}. We then consider a subset of the
attributes $C = C_1\cup C_2\cup C_3\cup C_4$ partitioned into four
groups. The first attribute group ($C_1$) consists of the voting
results for the political parties in 2009. The second attribute group
($C_2$) describes demographic properties such as the fraction of
elderly people, old people, middle aged people, young people, and
children in the population. The third group ($C_3$) contains
attributes describing the workforce in terms of the fraction of the
different professions such as agriculture, production, or service. The
fourth group ($C_4$) contains attributes describing the level of
education, unemployment and income. The attribute groupings are listed
in Table~\ref{tab:focus:cgs}. Thus, we here want to investigate
relations between different attribute groups, ignoring the relations
inside the groups.

We form the hypothesis pair $\mathcal{H}_{F, \emptyset}=\langle
\emptyset+\mathcal{T}_{F_1}, \emptyset+\mathcal{T}_{F_2} \rangle$,
where $\mathcal{T}_{F_1}$ consists of a tile spanning the rows in $R$
and the columns in $C$ whereas $\mathcal{T}_{F_2}$ consists of four
tiles: $t_i = (R, C_i)$, $i\in\{1,2,3,4\}$. Looking at the view in
which the distributions parametrised by the pair $\mathcal{H}_{F,
  \emptyset}$ differ the most, shown in
Figure~\ref{fig:exploration2}(a), we find two clear clusters
corresponding to a division of the districts into those located in the
East, and those located elsewhere. We could also have used our already
observed background knowledge of {\em Selection 1}, by considering the
hypothesis pair $\mathcal{H}_{F, \{t\}}=\langle
\{t\}+\mathcal{T}_{F_1}, \{t\}+\mathcal{T}_{F_2} \rangle$, where $t$
is the tile defined earlier for {\em Selection~1}.  For this
hypothesis pair, the most informative view is shown in
Figure~\ref{fig:exploration2}(b), which clearly is different to
Figure~\ref{fig:exploration2}(a), since we already were aware of the
relations concerning the rural districts in the East and this was
included in our background knowledge.

\begin{table}[t]
\centering
  \begin{tabular}{l l}
    \toprule
    \textbf{Group} & \textbf{Attributes} \\
\midrule \vspace{0.75ex}
$\bm{C_1}$  &  LEFT.2009, CDU.2009, SPD.2009, FDP.2009, GREEN.2009 \\ \vspace{0.75ex}
$\bm{C_2}$  & Elderly.pop., Old.Pop., Mid.aged.Pop., Young.Pop., Children.Pop.\\
$\bm{C_3}$  & Agricult..workf., Prod..workf.,  Manufac..Workf., Constr..workf., \\ \vspace{0.75ex}
            & Service.workf., Trade.workf.,  Finance.workf., Pub..serv..workf. \\
$\bm{C_4}$  & Highschool.degree, No.school.degree,  Unemploy., Unempl..Youth, Income\\
\bottomrule
  \end{tabular}
  \caption{Column groups in the focus tile in the exploration of the  \textsc{german} data set.}
  \label{tab:focus:cgs}
\end{table}

  \begin{figure}[t]
 \centering
\begin{tabular}{c@{}c}
\includegraphics[width=0.49\textwidth]{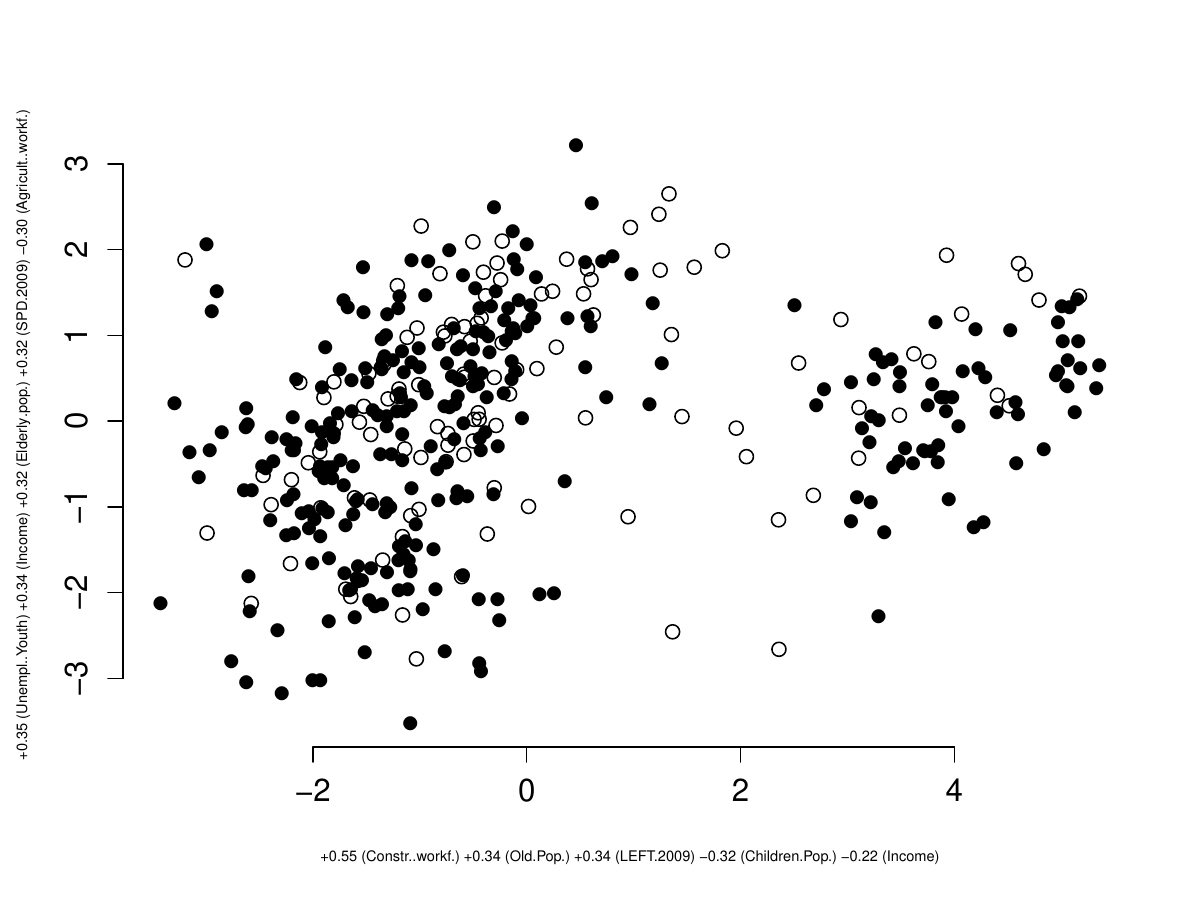} &
\includegraphics[width=0.49\textwidth]{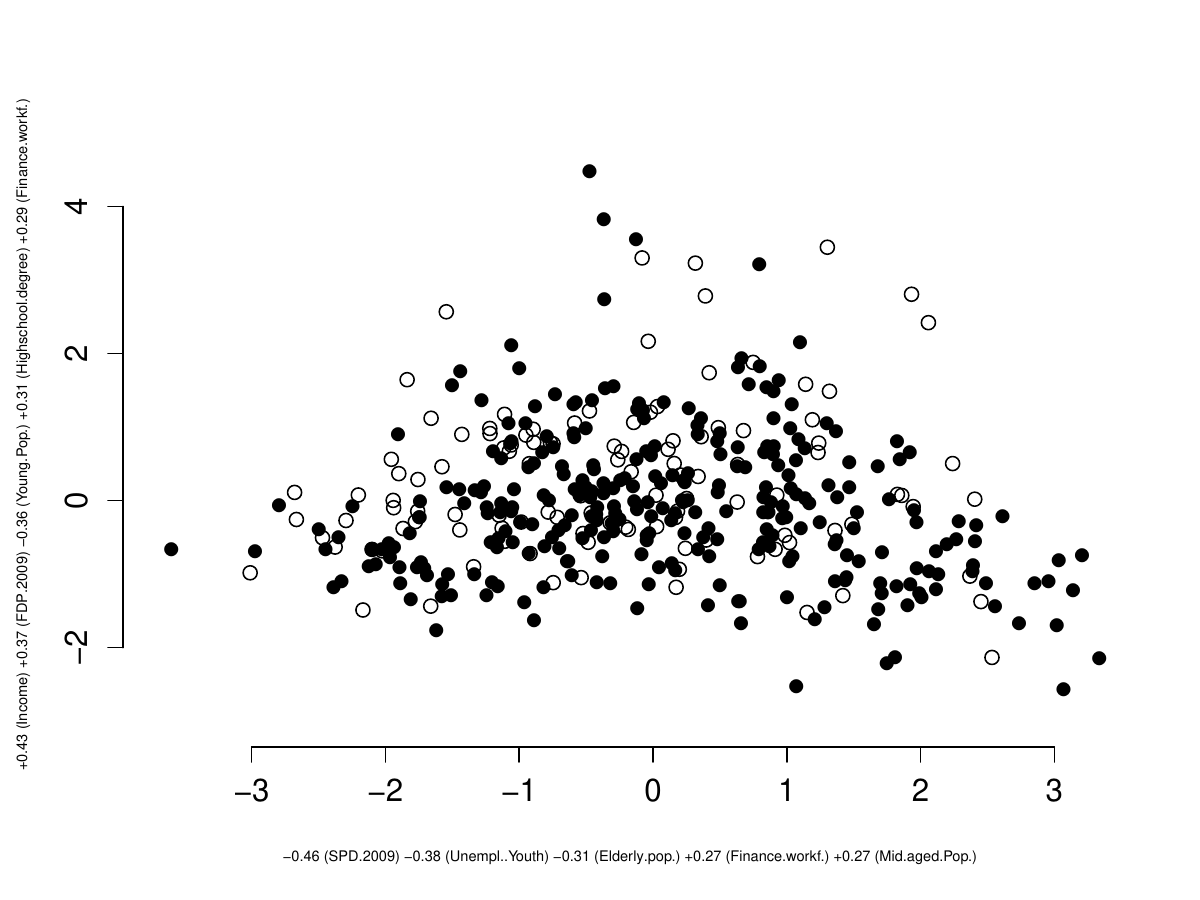}\\
(a) & (b)
\end{tabular}
  \caption{Views of the \textsc{german} data set corresponding to the
    hypothesis pairs (a) $\mathcal{H}_{F,\emptyset}$  and (b)
    $\mathcal{H}_{F,\{t\}}$. Data points inside the focus area
    are shown using filled circles ($\bullet$) and points outside the
    focus area are shown using hollow circles ($\circ$). The $x$ and
    $y$ axis labels show the five attributes with the largest absolute
    values in each projection vector.}
  \label{fig:exploration2}
 \end{figure}

 \begin{table}[t]
\centering
\begin{tabular}{l cccc}
\toprule
& $\bm{\mathcal{H}_{E,\emptyset}}$  &
  $\bm{\mathcal{H}_{E, \{t\}}}$     &
  $\bm{\mathcal{H}_{F, \emptyset}}$ &
  $\bm{\mathcal{H}_{F, \{t\}}}$ \\
\midrule
$\bm{v_{E, \emptyset}}$  & \textbf{8.831} & 4.211 & 1.921 &  1.195 \\
$\bm{v_{E, \{t\}}}$      & 8.105 & \textbf{8.875} &1.198 & 1.133 \\
$\bm{v_{F, \emptyset}}$  & 4.879 &  2.241 & \textbf{2.958} &  1.193 \\
$\bm{v_{F, \{t\}}}$      & 1.759 & 1.973 & 1.663 &  \textbf{1.762} \\
$\bm{v_\mathrm{pca}}$   & \textbf{8.831}& 4.211 & 1.921 &  1.195 \\
$\bm{v_\mathrm{ica}}$   & 0.005 & 0.005  & 1.000 &  0.999 \\
 \bottomrule
\end{tabular}
\caption{The value of the gain $G(v,\mathcal{H})$ for different
  projection vectors $v$ and hypothesis pairs $\mathcal{H}$.}
\label{tab:gains}
\end{table}

\paragraph{Comparison to PCA and ICA}
To demonstrate the utility of the views shown, we compute values of
the gain function as follows. We consider our four hypothesis pairs
$\mathcal{H}_{E, \emptyset}$, $\mathcal{H}_{E, \{t\}}$,
$\mathcal{H}_{F, \emptyset}$, and $\mathcal{H}_{F, \{t\}}$. For each
of these pairs, we denote the direction in which the two distributions
differ most in terms of the variance (solutions to
Equation~\eqref{eq:gain}) by $v_{E, \emptyset}$, $v_{E, \{t\}}$,
$v_{F, \emptyset}$, and $v_{F, \{t\}}$, respectively.  We then compute
the gain $G(v,\mathcal{H})$ for each $v\in \{v_{E, \emptyset}, v_{E,
  \{t\}}, v_{F, \emptyset}, v_{F, \{ t\}} \}$ and
$\mathcal{H}\in\{\mathcal{H}_{E, \emptyset}, \mathcal{H}_{E, \{t\}},
\mathcal{H}_{F, \emptyset}, \mathcal{H}_{F, \{t\}}\}$. For comparison,
we also compute the first principal component analysis (PCA) and
independent component analysis (ICA) \citep{hyvarinen1999fast}
projection vectors, denoted by $v_\mathrm{pca}$ and $v_\mathrm{ica}$,
respectively, and calculate the gain for different hypothesis pairs
using these. For ICA, we use the log-cosh~$G$ function and default
parameters of the R package {\texttt fastICA}.  The results are
presented in Table ~\ref{tab:gains}. We find that the gain is always
the highest when the projection vector matches the hypothesis pair
(highlighted in the table), as expected. This shows that the views
presented are indeed the most informative ones given the current
background knowledge and the hypothesis pair being investigated. We
also notice that the gain for PCA is equal to that of unguided data
exploration, as expected by Theorem~\ref{thm:pca}. When some
background knowledge is used or if we investigate a particular
hypothesis, the views with PCA or ICA objectives are less informative
than the one obtained using our framework. The gains close to zero for
the ICA objective are directions in which the variance of the more
constrained distribution is small due to, for example, linear
dependencies in the data.

% ------------------------------------------------------------
\subsection{Exploration of the Accident Data Set}\label{ssec:other}
% ------------------------------------------------------------
Due to the preprocessing, several columns in the \textsc{accident}
data set are used to encode the distinct categorical values in the
original data. If we now want to explore relationships between the
original variables, we can define a hypothesis pair, in which columns
corresponding to the same categorical attribute are grouped
together. We can thus investigate relations between attribute groups,
ignoring relations inside the groups.

\begin{table}[t]
\centering
  \begin{tabular}{l lcl}
    \toprule
    \textbf{Group} & \textbf{Attribute} & \textbf{\# Values} & \textbf{Interpretation} \\
\midrule
$C_1$  & AM1LK &  14 & Occupation of  victim \\
$C_2$  & EUSEUR & 8 & Days lost (severity of the accident) \\
$C_3$  & IKAL & 12 & Age of victim (5-year bins, except 0--14 and\\ & & &
 $>$65 years) \\
$C_4$  & NUORET & 4& Age of victim (0--15, 16--17, 18--19, and $>$19 years) \\
$C_5$  & RUUMIS & 33 &  Injured body part \\
$C_6$ & POIKKEA & 10& Deviation before accident\\
$C_7$ &  SATK & 12 & Month of  accident\\
$C_8$ &   SUKUP & 2 & Gender \\
$C_9$ &  TOLP & 22 & Main industry category \\
 $C_{10}$ &  TPISTE & 4 & 	Workstation \\
$C_{11}$ &  TYOSUOR & 9  & 	Specific physical activity \\
$C_{12}$ &   TYOTEHT & 31  & Work done at the time of  accident\\
$C_{13}$ &   VAHITAP &  16  & Contact-mode of injury \\
$C_{14}$ &      VAHTY & 2  & 	Accident class \\
$C_{15}$ &       VAMMAL & 13  &Type of injury\\
$C_{16}$ &       VPAIVA & 7  & Week day of  accident \\
$C_{17}$ &  VUOSI & 12  & Year of  accident\\
\bottomrule
  \end{tabular}
  \caption{Column groups in the focus tile in the exploration of the
    \textsc{accident} data set. For each attribute there are `\#
    Values' columns in the data set which are grouped together in the
    hypothesis tilings }
  \label{tab:focus:accs}
\end{table}

We define the hypothesis pair as follows. As the subset of rows $R$ we
choose all the 3000 rows in the \textsc{accident} data set. We then
consider a subset of the attributes $C = C_1\cup C_2\cup \cdots \cup
C_{17}$, where a summary of the attribute groupings
$C_1,\ldots,C_{17}$ is provided in Table~\ref{tab:focus:accs}.  The
hypothesis pair is then $\mathcal{H}_{A, \emptyset}=\langle
\emptyset+\mathcal{T}_{A_1}, \emptyset+\mathcal{T}_{A_2} \rangle$,
where $\mathcal{T}_{A_1}$ consists of a tile spanning all rows in $R$
and all columns in $C$ whereas $\mathcal{T}_{A_2}$ consists of tiles:
$t_i = (R, C_i)$, $i\in\{1,\ldots,17\}$.  The view in which the
distributions parametrised by $\mathcal{H}_{A ,\emptyset}$ differ the
most is shown in Figure~\ref{fig:exploration3}(a). Here we observe two
clear clusters. We select the points shown in purple and observe that
these data points correspond to accidents happening during travel to
work (\textsc{vahty.m}). Also, the attributes \textsc{tpiste},
\textsc{tyoteht}, \textsc{tyosuor}, \textsc{poikkea}, and
\textsc{vahitap} (all related to the type of work or the place of work
in which the accident occurred) have missing values (--) here, which
is natural when an accident happens during travelling to work. The
points in the complement of the purple selection on the other hand
correspond to accidents happening in the workplace (\textsc{vahty.p}).

Next, we add a tile $t_A$ with the selection of purple points in
Figure~\ref{fig:exploration3}(a) as the set of rows and all columns in
the {\sc accident} data as the set of columns to incorporate our knowledge
concerning these points into the exploration. We proceed to
consider the updated hypothesis pair $\mathcal{H}_{A, \{t_A\}}=\langle
\{t_A\}+\mathcal{T}_{A_1}, \{t_A\}+\mathcal{T}_{A_2} \rangle$. The
most informative view for this hypothesis pair is shown in
Figure~\ref{fig:exploration3}(b).  For illustration purposes we show
in purple the same selection of rows in data as in
Figure~\ref{fig:exploration3}(a). We can now select, for example, the
data points shown in orange in Figure~\ref{fig:exploration3}(b) for
further inspection. This selection corresponds to accidents in the
workplace (\textsc{vahty.p}) which happened mainly to women
(\textsc{sukup.n}) over 19 years old (\textsc{nuoret.4}), 
resulting in 31--90 days of absence from work (\textsc{euseur.9}).

\begin{figure}[t]
 \centering
\begin{tabular}{c@{}c}
\includegraphics[width=0.49\textwidth]{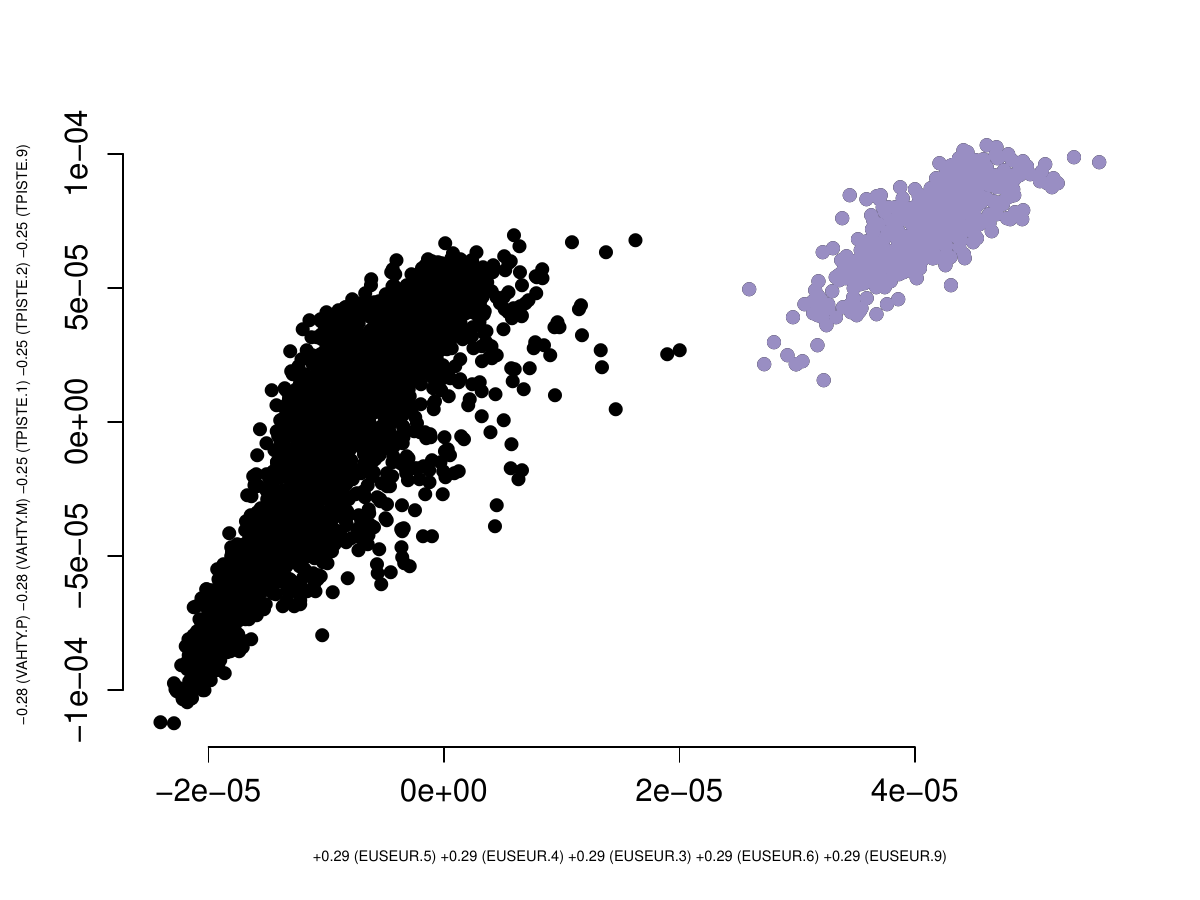} &
\includegraphics[width=0.49\textwidth]{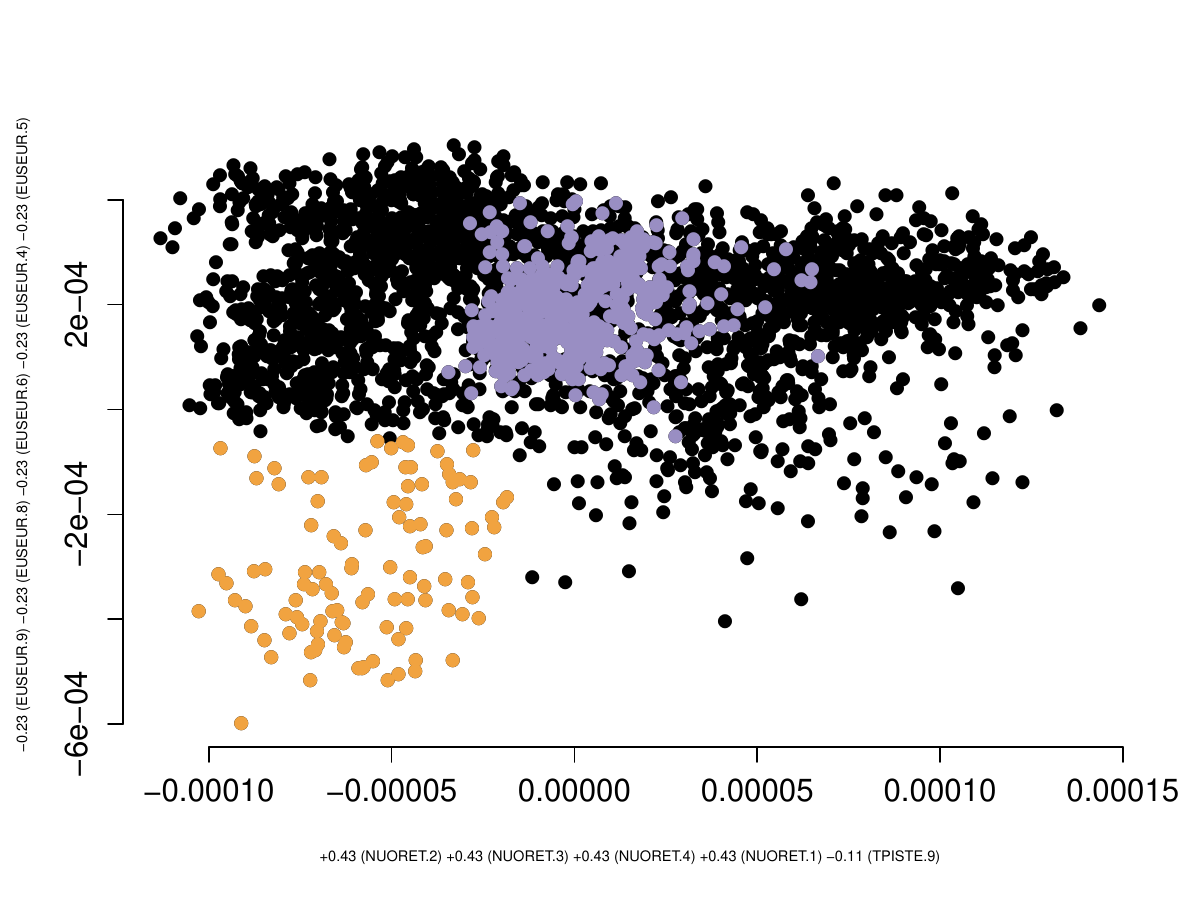} \\
(a) & (b)
\end{tabular}
  \caption{Views of the \textsc{accident} data set corresponding to
    the hypothesis pairs (a) $\mathcal{H}_{A,\emptyset}$ and (b)
    $\mathcal{H}_{A,\{t_A\}}$.  Filled circles show the data points;
    the selected points are marked with purple/orange, see description
    in the text.  The $x$ and $y$ axis labels show the five attributes
    with the largest absolute values in each projection vector.}
  \label{fig:exploration3}
 \end{figure}

% ------------------------------------------------------------
\section{Conclusions}\label{sec:conclusions}
% ------------------------------------------------------------
In this paper we propose an interactive visual data exploration
framework integrating the user's background knowledge (increasing
iteratively during the exploration) and the user's current exploration
interests in a principled way. We provide an efficient implementation
of this method using constrained randomisation. Furthermore, we also
extended PCA to work seamlessly with the framework in the case of
real-valued data sets.

Typical real-world data sets, for instance those empirically
investigated in this paper, contain a vast number of interesting
patterns. The goal of the data analyst is to find interesting
relations in the data. If automated analysis methods are used to
extract patterns, it means that the patterns must be specified in
advance to be used in conjunction with some data mining
algorithm. Specifying patterns in advance is clearly nontrivial when
there is a multitude of variable combinations that must be taken into
account. Furthermore, if patterns are only extracted based on a priori
specifications it is not possible to use insights obtained during the
exploration to steer further exploration.

It is here where the power of human-guided data exploration lies. A
non-interactive data mining method is restricted to either show
generic features of the data---which may already be obvious to an
expert---or output unusably many patterns, which is a typical problem,
for example, in frequent pattern mining (there are easily too many
patterns for the user to absorb). Our framework solves this problem:
by integrating the human's background knowledge and focus---formulated
as mathematically defined hypotheses---we can at the same time guide
the search towards topics interesting to the user at any particular
moment while taking the user's prior knowledge into account in an
understandable and efficient way. Hence, the framework described in
this paper makes it possible to interactively and efficiently explore
relations between attributes in the data through a conceptually simple
paradigm where the relations are encoded using tile constraints. This
exploration framework allows the data analyst to use his or her innate
pattern recognition skills to spot complex patterns, instead of having
to specify them in advance. As demonstrated in our empirical
evaluation of two real-world data sets, the proposed interactive
exploration framework allows us to find interesting patterns and hence
to make sense of the relations in the data.

Our work contains implicit assumptions about the human cognitive
processing, such that the user's knowledge can be modelled using a
background distribution. The validity of these assumptions is an
interesting topic for future research. For example, the order in which
different relations are observed probably matters to a real user,
whereas our formulation is invariant under ordering of the relations.
Also, the user is probably not able to model very fine-grained
distributions, while our mathematical formulation of the background
distribution can become extremely complex when the number of
constraints grows.

As a potential direction for future work we consider the extension of
the proposed method to understand classifiers or regression functions
in addition to static data.  Extending the ideas used here to
different data types such as, for example, time series, is also worth
investigating.  Finding an efficient algorithm that could find a
sparse solution to the optimisation problem of
Equation~\eqref{eq:gain} would also be an interesting problem. To the
best of our knowledge, no such solution is readily available. We note
that the solutions for sparse PCA are not directly applicable here:
sparse PCA would indeed give a sparse variant of the vector $w$ in
Theorem~\ref{thm:opt}. However, this would not result in a sparse
$v_{\mathcal{H}}=Ww$. Furthermore, we plan to study how to
incorporate in our framework a scheme for evaluating the statistical
significance of the visually observed patterns.

In our work, we chose to use a linear projection and showed that our
method reduces to PCA at the limit of generic objectives and no
background information. We showed that linear projections can be
computed efficiently for arbitrary background knowledge and 
objectives (expressed in our tile notation). We could, however, in
principle replace the linear projection by a non-linear embedding
that would in the same way show differences between the hypothesis pairs.
In their recent work, \citet{Kang2020} presented a variant of t-SNE
that can be used to produce informative visualisations with the
background information parametrised by a partition of data points into
known classes. Their method is not applicable to generic tile
constraints, but it might
be possible to use their ideas to develop a
non-linear embedding that would work with
generic tile constraints, in which case the resulting visualisation
method could be used as a drop-in replacement for the linear
projection presented in this paper.

In this paper we have implicitly assumed that the exploration takes
place in one exploration session by a single user. In practical
applications it might make sense to incorporate data constraints
learned on earlier exploration sessions or by other users into the
model, and to modify the hypotheses based on the insights gained.

Finally, we have implemented an open source R package that allows us
to simulate interactive visual data exploration using our
framework. The framework is available under an open source license
from \url{https://github.com/edahelsinki/corand/} and it includes, in
addition to the code needed to run the experiments in this paper, an
interactive web-based interface prototype. We have also earlier
released a preliminary prototype called {\sc tiler}
\citep{ecml-pkdd:tiler}, which includes the tile-based constrained
randomisation approach, but does not implement the dimensionality
reduction method presented in this work.

% ----------------------------------------------------------------------

\acks{We thank Buse Gul Atli for discussions and contributions to the
  preprint \citep*{puolamaki2018human}. We thank the Finnish Workers'
  Compensation Center for the access to the accident data. This work
  was supported by the Academy of Finland (decisions 326280 and
  326339).}
\appendix

% ----------------------------------------------------------------------
\section{Algorithm for merging tiles} \label{app:algo}
% ----------------------------------------------------------------------

Merging a new tile into a tiling where all tiles are non-overlapping
can be done efficiently using Algorithm~\ref{alg:merge}. We assume
that the starting point is always a non-overlapping set of tiles and
hence we only need to consider the overlap that the new tile has with
the tiles in the tiling. This is similar to the merging of statements
considered by \citet{kalofolias:2016}. The algorithm has two
steps. Let $\mathcal{T}$ be the current tiling and $t=(R,C)$ the new
tile to be added. In the first step (lines 1--11) we identify the
tiles in $\mathcal{T}$ with which $t$ overlaps, and in the second step
(lines 12--17) we resolve (merge) the overlap between $t$ and the
tiles identified in the previous step.

The first step proceeds as follows. An empty hash map is initialised
(line 1) to be used to detect overlap between columns of the tiles in
$\mathcal{T}$ and the new tile $t$. We proceed to iterate over each
row $R$ in the new tile (lines 2--11). Since $\mathcal{T}$ is a
tiling, all its tiles are non-overlapping. We can thus store
$\mathcal{T}$ in a matrix of the same size as the data matrix where
each element corresponds to the ID of the tile that covers that
position. With a slight abuse of notation, $\mathcal{T}$ in the
algorithm refers to such a matrix. Now, given a row $i \in R$ and a
set of columns $C$ (line 3) we then get the IDs of the tiles on row
$i$ with which $t$ overlaps. We store this in $K$. The hash map is
used to detect if this row has been seen before, that is, whether $K$
is a key in $S$ (line 4). If this is the first time this row is seen,
$K$ is used as the key for a new element in the hash map and $S(K)$ is
initialised to be a tuple (line 5). Elements in this tuple are
referred to by name, for instance, $S(K)_\mathrm{rows}$ gives the set
of rows associated with the key $K$, while $S(K)_\mathrm{id}$ gives
the set of tile IDs. On lines 6 and 7 we store the current row index
$S(K)_\mathrm{rows}$ and the unique tile IDs $S(K)_\mathrm{id}$ in the
tuple. If the row was seen before, the row set associated with these
tile IDs is updated (line 9). After this first step, the hash map $S$
contains tuples of the form (\emph{rows}, \emph{id}) where \emph{id}
specifies the IDs of the tiles with which $t$ overlaps at the rows
specified by \emph{rows}.

In the second step of the algorithm (lines 12--17), we first determine
the currently largest tile ID in use (line 12). After this we iterate
over the tuples in the hash map $S$. For each tuple we must update the
tiles having IDs $S(K)_\mathrm{id}$ and on line 14 we hence find the
columns associated with these tiles. After this, the IDs of the
affected overlapping tiles are updated (line 15), and the tile ID
counter is incremented (line 16). Finally, the updated tiling is
returned on line 18. The time complexity of the tile merging algorithm
is $\mathcal{O}{(n m)}$.

\begin{algorithm2e}[t]
 \SetKwInOut{Input}{input}
 \SetKwInOut{Output}{output}

 \Input{Tiling $\mathcal{T}$ as an $n \times m$ data matrix where an
   element is the ID of the tile to which it belongs, and a tile $t =
   (R, C)$.}
 \Output{$\mathcal{T} + \{t\}$ (the tiling in which
   $\mathcal{T}$ is merged with $t$).}

 $S \leftarrow \texttt{HashMap}$\;

 \For{$i \in R$} {
   $K \leftarrow \mathcal{T}(i, C)$\;
   \uIf { $K \notin \texttt{keys}(S) $}{
     $S(K) \leftarrow \texttt{Tuple}$\;
     $S(K)_\mathrm{rows} \leftarrow \{i\}$\;
     $S(K)_\textrm{id} \leftarrow \texttt{unique}( \mathcal{T}(i, C))$\;
   }      \Else{
     $S(K)_\textrm{rows} \leftarrow  S(K)_\textrm{rows} \cup \{i\}$\;
   }
 }

 $p_\mathrm{max} \leftarrow \max(\mathcal{T}(R,C))$\;
 \For{$K \in \texttt{keys}(S)$}{
   $C' = \left\{ c \mid \mathcal{T}(S(K)_\mathrm{rows}, c) \in S(K)_\textrm{id}  \right\}$ \;
   $\mathcal{T}\left(S(K)_\textrm{rows}, C'\right) \leftarrow p_\mathrm{max} + 1$ \;
   $p_\mathrm{max} \leftarrow p_\mathrm{max} + 1$\;
     }
\BlankLine
 \Return{$\mathcal{T}$}
 \BlankLine
  \BlankLine
\caption{\label{alg:merge} Merging a tile $t$ with the tiles in a
  tiling $\mathcal{T}$. The function \texttt{HashMap} denotes a hash
  map. The value in a hash map $H$ associated with a key $x$ is $H(x)$
  and $\texttt{keys}(H)$ gives the keys of $H$. The function
  \texttt{Tuple} creates a (named) tuple. An element $a$ in a tuple $w
  = (\mathrm{a}, \mathrm{b})$ is accessed as $w_\mathrm{a}$. The
  function \texttt{unique} returns the unique elements of an array.}
\end{algorithm2e}

%\bibliography{19-364}

\end{document}